\documentclass[conference]{IEEEtran}
\IEEEoverridecommandlockouts
\usepackage{cite}
\usepackage{amsmath,amssymb,amsfonts}
\usepackage{algorithm,algpseudocode}
\usepackage{graphicx}
\usepackage{textcomp}
\usepackage{xcolor}
\def\BibTeX{{\rm B\kern-.05em{\sc i\kern-.025em b}\kern-.08em
    T\kern-.1667em\lower.7ex\hbox{E}\kern-.125emX}}

\usepackage[utf8]{inputenc}
\usepackage[T1]{fontenc}
\usepackage{microtype}
\usepackage{amsthm}
\usepackage{mathtools}
\usepackage{bbm}
\usepackage{bm}
\usepackage{booktabs}
\usepackage{multirow}
\usepackage{enumitem}
\usepackage{algorithm}
\usepackage{url}

\newtheorem{theorem}{Theorem}
\newtheorem{lemma}{Lemma}

\newtheorem{assumption}{Assumption}

\theoremstyle{remark}
\newtheorem{remark}{Remark}

\newcommand{\R}{\mathbb{R}}

\newcommand{\HV}{\mathcal{H}} 
\newcommand{\Pset}{\mathcal{P}} 
\newcommand{\Aset}{\mathcal{A}} 

\newcommand{\K}{K}
\newcommand{\T}{T}
\newcommand{\dobj}{d}
\newcommand{\mprobe}{q}

\newcommand{\UCB}{\mathsf{UCB}}
\newcommand{\LCB}{\mathsf{LCB}}
\DeclareMathOperator{\conv}{conv}
\DeclareMathOperator{\Pareto}{Pareto}

\begin{document}

\title{Probe-then-Commit Multi-Objective Bandits: Theoretical Benefits of Limited Multi-Arm Feedback
}

\author{Ming Shi \\ Department of Electrical Engineering, University at Buffalo, Buffalo, NY}

\maketitle

\begin{abstract}
We study an online resource-selection problem motivated by multi-radio access selection and mobile edge computing offloading. In each round, an agent chooses among $K$ candidate links/servers (arms) whose performance is a stochastic $d$-dimensional vector (e.g., throughput, latency, energy, reliability). The key interaction is \emph{probe-then-commit (PtC)}: the agent may probe up to $q>1$ candidates via control-plane measurements to observe their vector outcomes, but must execute exactly one candidate in the data plane. This limited multi-arm feedback regime strictly interpolates between classical bandits ($q=1$) and full-information experts ($q=K$), yet existing multi-objective learning theory largely focuses on these extremes. We develop \textsc{PtC-P-UCB}, an optimistic probe-then-commit algorithm whose technical core is frontier-aware probing under uncertainty in a Pareto mode, e.g., it selects the $q$ probes by approximately maximizing a hypervolume-inspired frontier-coverage potential and commits by marginal hypervolume gain to directly expand the attained Pareto region. We prove a dominated-hypervolume frontier error of $\tilde{O} (K_P d/\sqrt{qT})$, where $K_P$ is the Pareto-frontier size and $T$ is the horizon, and scalarized regret $\tilde{O} (L_\phi d\sqrt{(K/q)T})$, where $\phi$ is the scalarizer. These quantify a transparent $1/\sqrt{q}$ acceleration from limited probing. We further extend to \emph{multi-modal probing}: each probe returns $M$ modalities (e.g., CSI, queue, compute telemetry), and uncertainty fusion yields variance-adaptive versions of the above bounds via an effective noise scale.
\end{abstract}


\section{Introduction}\label{sec:introduction}

Next-generation wireless and edge systems increasingly rely on online selection among multiple candidate network resources while meeting heterogeneous service requirements~\cite{lin2022leveraging}. For example, in multi-radio access technology (multi-RAT) selection, a user equipment (UE) chooses among candidate links (e.g., 5G New Radio, WiFi, unmanned aerial vehicle relay) whose instantaneous channel quality, contention, and scheduling delay fluctuate across slots~\cite{li2020multi}. For another example, in mobile edge computing (MEC) offloading, a device selects an edge server whose queueing delay, compute load, and radio access conditions jointly determine end-to-end latency and success probability~\cite{mach2017mobile}. These decisions are inherently \emph{multi-objective}. Optimizing a single scalar metric can yield operating points that violate service-level objectives (SLOs), or systematically sacrifice a ``weak'' KPI (key performance indicators) to gain another (e.g., sacrifice reliability to gain throughput).

Existing online-learning abstractions tend to focus on two extremes. Multi-objective bandits (MOB) observe only the vector outcome of the \emph{single} executed arm per slot~\cite{drugan2013designing,cao2025provably,weymark1981generalized,roijers2013survey}, while full-information online optimization and learning (and vector-payoff approachability) observes outcomes of \emph{all} $K$ arms each round~\cite{blackwell1956analog,hazan2016introduction}. However, wireless or edge systems often operate in a distinct intermediate regime enabled by the control plane. A UE can \emph{probe} a small set of candidates using channel state information reference signals (CSI-RS) and beam sweeping, beacon frames or round-trip time (RTT) pings, or queue/CPU telemetry, but it can execute only one link/server due to data-plane constraints (one transmission/offload per slot). This yields \emph{limited multi-arm feedback}, which is richer than standard MOB, but much cheaper than full information.

We formalize this interaction as a per-round \emph{Probe-then-Commit} (PtC) protocol. At each slot $t$, the learner selects a probe set $S_t$ with $|S_t|=q$, observes vector outcomes $\{\mathbf r_t(k)\}_{k\in S_t}$ (control-plane feedback), and then commits to one executed resource $k_t\in S_t$ whose outcome incurs realized system performance (data-plane execution). This model exposes an explicit \emph{optimization-feedback tradeoff}: increasing $q$ improves information and should accelerate learning, while probing consumes measurement and signaling budget. This PtC regime bridges MOB and full-information experts, but requires new algorithmic and analytical tools to handle vector objectives and frontier criteria under probe-limited feedback.

Wireless/edge systems often need to operate across multiple modes. For example, one needs to consider energy-saving against latency-critical. Thus, a learner should discover the Pareto frontier rather than only optimize one fixed operating point. Accordingly, we evaluate preference-free learning by a hypervolume-based frontier coverage metric, and we evaluate preference-based learning by scalarized regret for monotone concave utilities, e.g., fairness-sensitive aggregations.

PtC multi-objective learning couples probe design and execution in ways that do not arise in classical models. (i) \emph{Uncertainty-aware frontier exploration:} probe-set design must lift multiple KPIs while diversifying across frontier regions; (ii) \emph{Frontier accuracy under partial feedback:} coverage guarantees require controlling a set-valued error (hypervolume gap); (iii) \emph{Quantifying the value of limited probing:} theory should expose how $q$ sharpens rates beyond the bandit regime; (iv) \emph{Multi-modal sensing:} probing often returns multiple modalities (CSI, queue length, CPU load) with heterogeneous noise, necessitating fusion that preserves valid confidence bounds~\cite{du2024distributed}.

Our main contributions are summarized as follows.
\begin{itemize}[leftmargin=*]
    \item \textbf{PtC multi-objective multi-feedback model (Sec.~\ref{sec:model}).}
    We introduce a stochastic multi-objective MAB under the PtC protocol with probe budget $q$. We formalize two complementary evaluation metrics, preference-free frontier learning via a dominated-hypervolume coverage gap and preference-based learning via scalarized regret for monotone concave utilities.

    \item \textbf{Algorithmic ideas: uncertainty-aware frontier coverage from probed samples (Sec.~\ref{sec:algorithm}).} We develop \textsc{PtC-P-UCB} (Algorithm~\ref{alg:ptc_p_ucb}), which elects the $q$ probes by approximately maximizing a hypervolume-inspired frontier-coverage potential and commits by marginal hypervolume gain to directly expand the attained Pareto region. The design is compatible with both frontier-coverage evaluation (hypervolume) and preference-based operation when a scalarizer is specified.

    \item \textbf{Theory: explicit value of limited multi-arm feedback (Sec.~\ref{sec:theory}).} We prove a dominated-hypervolume frontier coverage gap that vanishes at rate $\tilde{O} (K_P d/\sqrt{qT})$ (with frontier size $K_P$, $d$ objectives, and $T$ rounds), and a scalarized regret bound of order $\tilde{O} \big(d\sqrt{(K/q)T}\big)$ for monotone $L_\phi$-Lipschitz concave scalarizers, where $\tilde{O}$ hides constants and logarithmic terms. These results quantify a clean \emph{$1/\sqrt{q}$ improvement} from limited multi-arm probing, interpolating between the bandit limit ($q=1$) and the full-information limit ($q=K$).

    \item \textbf{Multi-modal extension with variance-adaptive gains (Sec.~\ref{sec:multimodal}).} We extend the framework to bundled multi-modal probing with $M$ modalities. We develop \textsc{MM-PtC-P-UCB} (Algorithm~\ref{alg:mm_ptc_p_ucb}) and show that fusion tightens confidence bounds through an effective noise scale, yielding variance-adaptive improvements
    for both frontier coverage and regret.

    \item \textbf{Empirical validation.}
    Simulations on multi-RAT- and MEC-inspired instances corroborate our theory. Modest probing budgets (e.g., $q\in\{2,4\}$) significantly accelerate learning and improve Pareto coverage with moderate overhead, and multi-modal fusion provides an additional orthogonal gain.
\end{itemize}

\subsection{Related Work}\label{subsec:related}

\subsubsection{Multi-objective bandits with Pareto criteria and scalarizations}

Multi-objective bandits study vector-valued rewards and compare actions via Pareto optimality or preference scalarizations. Upper confidence bound (UCB)-style approaches for vector rewards and Pareto efficiency appear in early MOB work (e.g.,~\cite{drugan2013designing}), while preference-based learning uses monotone utilities, e.g., inequality-averse aggregations such as generalized Gini~\cite{weymark1981generalized}, to encode fairness-sensitive tradeoffs and enable regret analysis~\cite{cao2025provably,agarwal2022multi}. More recent work develops Pareto-oriented regret notions that avoid fixing a scalarizer \cite{garivier2024sequential,roijers2013survey}. Our focus is complementary. We study an intermediate \emph{probe-limited multi-feedback} regime induced by wireless probing protocols, and provide guarantees for both frontier coverage (hypervolume) and preference-based regret under PtC feedback.

\subsubsection{Multiple-play, semi-bandits, and side-observation models}

Observing multiple actions per round relates to power-of-2-arms~\cite{shi2025power}, multiple-play bandits \cite{anantharam2003asymptotically}, combinatorial/semi-bandit models \cite{kveton2015tight}, and learning with structured side observations such as feedback graphs \cite{alon2015online}. These works are predominantly \emph{single-objective} and typically assume additive reward/loss decompositions, whereas PtC enforces a single executed action with vector outcomes and Pareto criteria. This changes both the algorithmic goal (probe-set design for frontier coverage) and the analysis (simultaneously controlling frontier estimation error and preference-based regret under probe-limited sampling).

\subsubsection{Vector-payoff online learning and approachability}

Full-information vector-payoff learning connects to Blackwell approachability \cite{blackwell1956analog} and its equivalence to no-regret learning \cite{abernethy2011blackwell}. This regime also includes standard online convex optimization with function information \cite{hazan2016introduction,shi2021competitive}. Our PtC model in this paper can be viewed as a probe-limited, partial-information counterpart tailored to wireless/edge measurement pipelines.

\subsubsection{Hypervolume as a Pareto-set quality measure}

Dominated hypervolume is a standard quality measure for Pareto sets in multi-objective optimization and evolutionary computation~\cite{zitzler2002multiobjective,auger2009theory}. We adopt a hypervolume gap as a principled metric for frontier discovery over time under PtC feedback.

\subsubsection{Multi-modal sensing and fusion in wireless/edge decision making}

Next-generation wireless and edge platforms expose heterogeneous \emph{modalities} correlated with service quality, including radio measurements, active probes, and system telemetry (e.g., queue/compute load). Recent wireless research also emphasizes multi-modal learning at the network level, including foundation-model perspectives for 6G systems~\cite{du2024distributed}. To our knowledge, we take the first effort to model multi-modal probing as \emph{multiple noisy views} of the same underlying multi-objective outcome vector, and design confidence-bound-driven learning rules whose uncertainty tightens via an effective variance under fusion. This yields variance-adaptive improvements in learning performance, while preserving the explicit $1/\sqrt{\mprobe}$ benefit of limited multi-arm probing under the PtC protocol.

\textbf{Notation:} For an integer $n$, $[n]=\{1,2,\ldots,n\}$. For a scalar $x$, $[x]^+\triangleq \max\{0,x\}$. For $u,v\in\mathbb{R}^{d}$, $u\succeq v$ denotes component-wise inequality and $\|u\|_\infty=\max_j |u_j|$. We write $\Delta_d=\{w\in\mathbb{R}_+^{d}:\sum_{j=1}^d w_j=1\}$ for the probability simplex.

\section{Problem Formulation}\label{sec:model}

We study a multi-objective online esource-selection problem motivated by wireless access and edge computing systems, where each decision must balance heterogeneous KPIs under limited probing and measurement opportunities. The key interaction is \emph{probe-then-commit} (PtC): in each round, the agent may probe multiple candidates via control-plane feedback, but ultimately executes only one due to data-plane constraints.

\subsection{System Model and PtC Feedback Protocol}\label{subsec:model_protocol}

There are $\K$ candidate resources (arms), indexed by $k\in[\K]$ (e.g., access links, edge servers), and $\dobj$ objectives indexed by $j\in[\dobj]$ (e.g., throughput, latency, energy, reliability). Time is slotted with horizon $\T$. At each round $t\in[\T]$, each arm $k$ is associated with a random vector-valued reward outcome
\begin{align}
\mathbf r_t(k)=\big(r_t^{(1)}(k),\ldots,r_t^{(\dobj)}(k)\big)\in[0,1]^{\dobj}, \label{eq:vec_outcome}
\end{align}
with \emph{unknown} mean $\boldsymbol\mu(k)=\mathbb E[\mathbf r_t(k)]\in[0,1]^{\dobj}$. The $\dobj$ coordinates represent heterogeneous KPIs. If a KPI is naturally minimized (e.g., delay/energy), we convert it to a maximization objective by negating and normalizing it to $[0,1]$.

\begin{algorithm}[t]
\caption{Probe-then-Commit (PtC) interaction at round $t$}\label{alg:ptc_protocol}
\begin{algorithmic}[1]
\Require Probe budget $\mprobe\in[\K]$.
\State \textbf{Probe selection:} choose a probe set $S_t\subseteq[\K]$ with $|S_t|=\mprobe$.
\State \textbf{Measurement:} observe vector outcomes $\{\mathbf r_t(k)\}_{k\in S_t}$.
\State \textbf{Commit:} select one executed arm $k_t\in S_t$.
\State \textbf{Realization:} incur the round performance $\mathbf r_t(k_t)$ and proceed to next round $t{+}1$.
\end{algorithmic}
\end{algorithm}

In many practical scenarios (e.g., multi-RAT selection and MEC offloading), a device can probe and measure multiple candidates (e.g., channel probing, active RTT pings, queue/CPU reports), but can use only one for actual transmission/offloading in that round. We model this by PtC (see Algorithm~\ref{alg:ptc_protocol}).

Probe outcomes for all $k\in S_t$ are observed (control-plane measurement), but only the executed arm $k_t$ contributes to realized system performance in that round (data-plane execution). Probing may incur overhead (time/energy/signaling). We treat $\mprobe$ as a fixed per-round budget in the main theory and discuss explicit probing-cost models in Sec.~\ref{subsec:model_metrics}.

\subsection{Pareto Structure and Preference Scalarizers}\label{subsec:model_pareto}

The vector nature of $\boldsymbol\mu(k)$ induces a partial order, i.e., an arm may be better in one KPI but worse in another. We therefore formalize preference-free efficiency through Pareto dominance and, when the system operates under a fixed preference, use scalarizers to select a single operating point.

\subsubsection{Pareto dominance and frontier}

For $u,v\in\mathbb R^{\dobj}$, we say that $u$ dominates $v$ (denoted $u\succ v$) if $u_j\ge v_j$ for all $j$ and $u_{j'}>v_{j'}$ for some $j'$. Given mean vectors $\{\boldsymbol\mu(k)\}_{k=1}^{\K}$, define the Pareto frontier $\Pset^*=\mathrm{Pareto}\big(\{\boldsymbol\mu(k)\}_{k=1}^{\K}\big)$, i.e., the set of mean vectors not dominated by any other. Moreover, we let $K_P\triangleq|\Pset^*|$ denote the size of the true frontier $\Pset^*$.

\subsubsection{Scalarizers with system preferences}

A deployed system may need a particular tradeoff based on operator policy or user preference. We model such preferences by scalarizers $\phi:\mathbb R^{\dobj}\to\mathbb R$ that are \emph{monotone} (improving any objective cannot decrease utility), \emph{$L_\phi$-Lipschitz} w.r.t.\ $\ell_\infty$ (i.e., $|\phi(u)-\phi(v)|\le L_\phi\|u-v\|_\infty$), and \emph{concave}. These conditions enable stable aggregation and regret analysis. Standard examples include: (i) weighted sum $\phi_w(u)=\sum_{j=1}^{\dobj} w_j u_j$ with $w\in\Delta_{\dobj}$; (ii) Chebyshev scalarizer $\phi^{\min}_w(u)=\min_{j\in[\dobj]} w_j u_j$; and (iii) generalized Gini $\phi_\gamma(u)=\sum_{i=1}^{\dobj}\gamma_i u_{(i)}$, where $u_{(1)}\le\cdots\le u_{(\dobj)}$ are the sorted components and $\gamma_1\ge\cdots\ge\gamma_{\dobj}\ge0$~\cite{weymark1981generalized,drugan2013designing}.

\subsection{Performance Metrics and Probe Overhead}\label{subsec:model_metrics}

We evaluate algorithms using two complementary criteria. \emph{Hypervolume coverage} quantifies preference-free frontier learning performance, while \emph{scalarized regret} measures learning and operational loss under a specified system preference.

\subsubsection{Hypervolume-based Pareto coverage}

To quantify preference-free frontier learning, we use dominated hypervolume (HV) with respect to a fixed reference point $z_{\mathrm{ref}}\in\mathbb R^{\dobj}$ that is component-wise worse than all attainable performance vectors. For compact set $\mathcal S\subseteq\mathbb R^{\dobj}$, define the dominated region
\begin{align}
\mathcal D(\mathcal S)\triangleq\left\{y\in\mathbb R^{\dobj}: \exists u\in\mathcal S, \text{ s.t. } z_{\mathrm{ref}}\preceq y\preceq u\right\},
\end{align}
and let $\HV(\mathcal S)$ be the Lebesgue measure of $\mathcal D(\mathcal S)$.
Intuitively, larger $\HV(\mathcal S)$ means that $\mathcal S$ contains operating points that jointly perform well across objectives and spans a broader range of Pareto tradeoffs. Then, we let $\mathcal Y_T\triangleq\{\mathbf r_t(k_t)\}_{t=1}^{T}$ be the archive of executed outcome vectors and define the attained set
\begin{align}
\Aset_T\triangleq \mathrm{conv}(\mathcal Y_T),
\end{align}
where the convex hull captures time-sharing among operating points. Since the same time-sharing interpretation applies to the Pareto benchmark, define the convexified Pareto set $\mathcal C^*\triangleq \mathrm{conv}(\Pset^*)$. We measure the remaining uncovered dominated volume by the attained-set hypervolume gap~\cite{zitzler2002multiobjective,auger2009theory}
\begin{align}
\mathcal L_T^{\mathrm{HV}} \triangleq \left[\HV(\mathcal C^*)-\HV(\Aset_T)\right]^+. \label{eq:hv_attained_gap}
\end{align}
By construction, $\mathcal L_T^{\mathrm{HV}}\ge0$. Smaller $\mathcal L_T^{\mathrm{HV}}$ indicates that the executed decisions achieve tradeoffs whose dominated region approaches that of the (time-shareable) Pareto benchmark.

\subsubsection{Scalarized regret}

Fix a scalarizer $\phi$ and define the best arm in hindsight $k^*\triangleq\arg\max_{k\in[\K]}\phi(\boldsymbol\mu(k))$.
Then, the scalarized regret is defined as follows,
\begin{align}
R_\T^{\phi} = \sum\nolimits_{t=1}^{\T}\left(\phi(\boldsymbol\mu(k^*))-\phi(\mathbf r_t(k_t))\right). \label{eq:scalarized-regret}
\end{align}

\subsubsection{Probe overhead}

Probing consumes control-plane resources (time and energy). A simple model assigns a per-probe cost $\tau>0$ and penalizes each round by $-\tau|S_t|=-\tau\mprobe$. Equivalently, one may impose a hard constraint $\mprobe\le M_{\max}$ or a long-term budget. Our main results treat $\mprobe$ as fixed and quantify how increased probing improves learning rates.

\section{Motivating Examples}\label{sec:motivating}

We highlight two wireless/edge scenarios that directly match the PtC protocol in Algorithm~\ref{alg:ptc_protocol}: a device can obtain control-plane measurements from up to $\mprobe$ candidates within a slot, but can execute only one candidate on the data plane.

\subsubsection{Multi-RAT link selection}

A UE probes up to $\mprobe$ candidate access points or gNBs using lightweight control-plane signals (e.g., pilot measurements, beacons, short RTT probes, or queue indicators), obtaining a KPI vector such as throughput, delay, energy, and reliability. It then commits to one link for data transmission, so only the executed link's outcome is realized.

\subsubsection{MEC offloading}

A device queries up to $\mprobe$ MEC servers for multi-modal telemetry (e.g., radio quality, queue status, CPU load), forming an outcome vector that captures end-to-end latency, energy consumption, and SLO satisfaction/reliability. It then offloads to a single server, again matching PtC.

\section{Algorithm Design}\label{sec:algorithm}

\begin{algorithm}[t]
\caption{\textsc{PtC-P-UCB}: Probe-then-Commit Pareto-UCB}\label{alg:ptc_p_ucb}
\begin{algorithmic}[1]
\Require Probe budget $\mprobe$, weights $w\in\Delta_{\dobj}$, reference point $z_{\mathrm{ref}}$, confidence parameter $\{\beta_t\}$. Commit mode: \textsc{Scalar} (use given $\phi$) or \textsc{HV} (use marginal hypervolume gain)
\State Initialize $N_1(k)\gets 0$, $\widehat{\mu}_1^{(j)}(k)\gets 0$ for all $k\in[\K]$ and $j\in[\dobj]$, active set $\mathcal K_1\gets[\K]$ and executed archive $\mathcal Y_0\gets\emptyset$
\For{$t=1$ to $\T$}
    \Statex \hspace{0.5em}\textbf{Confidence bounds (from probed samples):}
    \For{each $k\in\mathcal K_t$ and $j\in[\dobj]$}
        \State Calculate $b_t^{(j)}(k)$ using \eqref{eq:conf_radius}
        \State Calculate $\UCB_t^{(j)}(k)$ and $\LCB_t^{(j)}(k)$ using \eqref{eq:ucblcbdefine}
    \EndFor
    \State Define $u_t(k)$ and $\ell_t(k)$ for all $k\in\mathcal K_t$

    \Statex \hspace{0.5em}\textbf{Pruning:}
    \State $\widetilde{\mathcal K}_t \gets \{k\in\mathcal K_t: \nexists k'\in\mathcal K_t \text{ s.t. } \ell_t(k') \succ u_t(k)\}$
    \State $\mathcal K_t \gets \widetilde{\mathcal K}_t$

    \Statex \hspace{0.5em}\textbf{Probe selection:}
    \If{\textsc{HV}}
    \Comment{\textsc{HV}}
    \State Select $S_t\subseteq\mathcal K_t$ by maximization of $F_t(S)$ in \eqref{eq:coverage_potential}
    \Else
    \Comment{\textsc{Scalar}}
    \State Select $S_t\subseteq\mathcal K_t$ by the surrogate top-$\mprobe$ rule \eqref{eq:phi_ucb_score}
    \EndIf
    \State Probe and observe $\{\mathbf r_t(k)\}_{k\in S_t}$

    \Statex \hspace{0.5em}\textbf{Commit:}
    \If{\textsc{HV}}
    \Comment{\textsc{HV}}
    \State $k_t \gets \arg\max_{k\in S_t} \Delta_t^{\HV}(k)$ using \eqref{eq:marginal_hv}
    \Else 
    \Comment{\textsc{Scalar}}
    \State $k_t \gets \arg\max_{k\in S_t} \phi(\mathbf r_t(k))$
    \EndIf
    \State Execute $k_t$ and incur $\mathbf r_t(k_t)$
    \If{\textsc{HV}}
    \State Update archive $\mathcal Y_t\gets \mathcal Y_{t-1}\cup\{\mathbf r_t(k_t)\}$
    \EndIf

    \Statex \hspace{0.5em}\textbf{Updates (for all probed arms):}
    \For{each $k\in S_t$}
        \State $N_{t+1}(k)\gets N_t(k)+1$
        \For{each $j\in[\dobj]$}
            \State $\widehat{\mu}_{t+1}^{(j)}(k)\gets \widehat{\mu}_t^{(j)}(k) + \frac{r_t^{(j)}(k)-\widehat{\mu}_t^{(j)}(k)}{N_{t+1}(k)}$
        \EndFor
    \EndFor
    \State For $k\notin S_t$, $N_{t+1}(k)\gets N_t(k)$, $\widehat{\mu}_{t+1}^{(j)}(k)\gets \widehat{\mu}_t^{(j)}(k)$.
    \State Set $\mathcal K_{t+1}\gets \mathcal K_t$.
\EndFor
\end{algorithmic}
\end{algorithm}

This section presents \textsc{PtC-P-UCB} (see Algorithm~\ref{alg:ptc_p_ucb}), a probe-then-commit algorithm for multi-objective learning under PtC feedback. The algorithm maintains arm-wise confidence bounds from probed samples, selects a probe set to accelerate learning, and then commits to one probed arm for execution.

The PtC protocol requires executing exactly one arm per round. Our hypervolume coverage metric in~\eqref{eq:hv_attained_gap} is defined on the executed attained set $\Aset_T$, and therefore depends on which arm is committed each round. To obtain preference-free frontier coverage for $\Aset_T$, we use a coverage-aware commit rule based on marginal hypervolume gain. However, if the system is operated under a known preference, the commit step can maximize a scalarizer $\phi$ to minimize scalarized regret.

Our design follows three main principles as follows.
\begin{enumerate}
[leftmargin=*]
    \item \emph{Probe selection should increase frontier coverage under uncertainty.} Since probing provides the side information ($\mprobe$ observations per round), the probe set should \textbf{include} arms that are plausibly Pareto-efficient under optimism and \textbf{diversify} across frontier regions. We achieve this by approximately maximizing a hypervolume-based coverage potential over optimistic vectors, with a score-based surrogate.

    \item \emph{Multi-objective optimism with component-wise confidence.} Maintain per-objective confidence intervals and form optimistic vectors to guide probing and safe elimination.

    \item \emph{Commit must match the metric.} For preference-free frontier learning under~\eqref{eq:hv_attained_gap}, we commit using marginal hypervolume gain of the executed archive. For preference-based operation we commit using $\phi$ to minimize scalarized regret.
\end{enumerate}

\subsection{Preliminaries for Confidence Bounds}\label{subsec:alg_confidence}

Since the learner observes outcomes for all probed arms, we index learning progress by the number of times an arm has been probed, $N_t(k) \triangleq \sum_{\tau=1}^{t-1}\mathbb{I}\{k\in S_\tau\}$. For each objective $j\in[\dobj]$, maintain the empirical mean based on probed samples,
\begin{align}
\widehat{\mu}_t^{(j)}(k) \triangleq \frac{1}{\max\{1,N_t(k)\}}\sum\nolimits_{\tau=1}^{t-1}\mathbb{I}\{k\in S_\tau\} r_\tau^{(j)}(k).
\end{align}
Choose a confidence parameter $\beta_t = 2\log(2\K\dobj t^2/\delta)$ for Hoeffding-style bounds. Define the bonus term for each $(k,j)$,
\begin{align}
b_t^{(j)}(k) \triangleq \sqrt{ \beta_t / \max\{1,N_t(k)\}}. \label{eq:conf_radius}
\end{align}
Then, we form clipped upper/lower bounds
\begin{align}
\UCB_t^{(j)}(k) & = \min\left\{1,\widehat{\mu}_t^{(j)}(k)+b_t^{(j)}(k)\right\}, \nonumber \\
\LCB_t^{(j)}(k) & = \max\left\{0,\widehat{\mu}_t^{(j)}(k)-b_t^{(j)}(k)\right\}. \label{eq:ucblcbdefine}
\end{align}
We define the optimistic and pessimistic vectors
\begin{align}
& u_t(k) \triangleq \big(\UCB_t^{(1)}(k),\ldots,\UCB_t^{(\dobj)}(k)\big), \nonumber \\
& \ell_t(k) \triangleq \big(\LCB_t^{(1)}(k),\ldots,\LCB_t^{(\dobj)}(k)\big).
\end{align}
With high probability, $\ell_t(k)\preceq \mu(k)\preceq u_t(k)$ component-wise for all $k,t$, enabling safe pruning and optimistic probe selection.

\subsection{Probe Selection via Frontier-Coverage Potential}\label{subsec:alg_probe}

A key decision is how to choose the probe set $S_t$ of size $\mprobe$. Selecting the top-$\mprobe$ arms by a single scalar score may overly concentrate probing around one region of the frontier. To encourage \emph{diverse} coverage, we define a set-based potential.

\subsubsection{Coverage potential (set function)}

Let $z_{\mathrm{ref}}$ be the hypervolume reference point (as in Section~\ref{subsec:model_metrics}). Given optimistic vectors $\{u_t(k)\}_{k=1}^{\K}$, define the potential of a probe set $S$ as
\begin{align}
F_t(S) \triangleq \HV \left(\mathrm{conv}\{u_t(k)\}_{k\in S}\right), \label{eq:coverage_potential}
\end{align}
i.e., the dominated hypervolume of the convexified optimistic set. This potential rewards probe sets whose optimistic vectors jointly dominate a large region, which aligns with minimizing the Pareto coverage gap (up to optimism and estimation error).

\subsubsection{Greedy probe selection}

Maximizing $F_t(S)$ over all $|S|=\mprobe$ is combinatorial. We therefore use a greedy approximation. Starting from $S=\emptyset$, iteratively add the arm with the largest marginal gain in $F_t$. When $F_t(\cdot)$ is (approximately) monotone submodular, greedy achieves a constant-factor approximation.

\subsubsection{Fast surrogate (general scalarizer)}

When scalarized regret is the metric, we use a modular surrogate obtained by applying a preference scalarizer to the optimistic vector, i.e.,
\begin{align}
\mathrm{score}_t^{\phi}(k)\triangleq \phi \big(u_t(k)\big). \label{eq:phi_ucb_score}
\end{align}
We then set $S_t$ to be the $\mprobe$ arms in $\mathcal K_t$ with largest $\mathrm{score}_t^{\phi}(k)$.

\subsection{Commit Rule for Execution}\label{subsec:alg_commit}

After probing, the learner observes $\{\mathbf r_t(k)\}_{k\in S_t}$ and must execute one arm $k_t\in S_t$. We provide two commit rules that correspond to the two evaluation objectives.

\subsubsection{Coverage-based commit (for hypervolume coverage on $\Aset_T$)}

Let $\mathcal{Y}_{t-1}$ denote the archive of executed outcomes up to round $t-1$ (so that
$\Aset_{t-1}=\mathrm{conv}(\mathcal{Y}_{t-1})$).
For each candidate $k\in S_t$, define the marginal hypervolume gain
\begin{align}
\Delta_t^{\HV}(k) \triangleq \HV \left( \mathrm{conv} \left( \mathcal{Y}_{t-1}\cup\{\mathbf r_t(k)\} \right) \right) - \HV \left( \mathrm{conv}(\mathcal{Y}_{t-1}) \right). \label{eq:marginal_hv}
\end{align}
We then commit to the arm that optimizes this gain, i.e.,
\begin{align}
k_t^{\HV} \in \arg\max\nolimits_{k \in S_t} \Delta_t^{\HV}(k). \label{eq:commit_hv}
\end{align}
Intuitively, this chooses the probed arm that most expands the dominated region of the attained set, directly targeting $\mathcal L_T$.

\subsubsection{Preference-based commit (for scalarized regret)}

Given a preference scalarizer $\phi$, we commit to the best observed arm
\begin{align}
k_t^{\phi} \in \arg\max\nolimits_{k\in S_t} \phi (\mathbf r_t(k)). \label{eq:commit_phi}
\end{align}
This extracts immediate operational value from probing.

\subsection{Frontier Pruning}\label{subsec:alg_prune}

To concentrate probing on plausible Pareto-frontier arms, we maintain an active set $\mathcal K_t\subseteq[\K]$. An arm $k$ can be safely discarded if it is certifiably dominated, i.e., $\exists k' \in \mathcal K_t$, s.t. $\ell_t(k') \succ u_t(k)$. On the high-probability event $\ell_t(\cdot)\preceq\mu(\cdot)\preceq u_t(\cdot)$, this rule never removes a true Pareto arm, but can dramatically reduce computation and improve the dependence on the frontier size $K_P$ in the coverage analysis. Moreover, since $\mu(k')\succeq \mu(k)$ implies $\phi(\mu(k'))\ge \phi(\mu(k))$ for monotone scalarizer $\phi$, the rule also never removes a $\phi$-optimal arm.

\subsection{Complexity and Probe Overhead}\label{subsec:alg_complexity}

Per round, computing confidence bounds costs $O(|\mathcal K_t|\dobj)$. Probe selection costs $O(\mprobe\,|\mathcal K_t|\,C_{\HV})$ under greedy hypervolume or $O(\K\log\K)$ under the score-based surrogate, where $C_{\HV}$ is the cost of a hypervolume marginal computation (small for $\dobj\le 4$). The probing overhead scales linearly with $\mprobe$ in control-plane signaling. A per-probe cost $\tau$ can be incorporated by constraining $\mprobe\le M_{\max}$ or subtracting $\tau$ from the scalar utility in deployment, without changing the learning machinery.

\section{Theoretical Results}\label{sec:theory}

This section provides performance guarantees for \textsc{PtC-P-UCB} (Algorithm~\ref{alg:ptc_p_ucb}) under the stochastic PtC model for the two evaluation metrics in Sec.~\ref{subsec:model_metrics}. We provide the proof sketches for main theorems, while complete proofs and supporting lemmas are provided in appendices and our technical report~\cite{mingshihomepagetechrptwiopt}.

\subsection{Assumption and Key Bookkeeping}\label{subsec:theory_assumptions}

Let $\mathcal{F}_{t}$ be the filtration generated by past probe sets and observed probes up to the end of round $t$. A central identity is
\begin{align}
\sum\nolimits_{k=1}^{\K} N_{T+1}(k) = \mprobe T, \label{eq:probe_budget_identity}
\end{align}
i.e., each round yields $\mprobe$ vector samples. Compared with traditional bandits, the learner observes $\mprobe$ times more arm outcomes per round thanks to probes, which shrinks estimation error faster and translates into improved learning rates.

\begin{assumption}[Conditionally sub-Gaussian noise]\label{assm:sub-Gaussiannoise}
For each $k\in[\K]$ and $j\in[\dobj]$, the noise $r_t^{(j)}(k)-\mu^{(j)}(k)$ is conditionally $\sigma$-sub-Gaussian given $\mathcal{F}_{t-1}$ and independent over round $t$.
\end{assumption}

\subsection{Preference-Free Pareto Frontier Learning}\label{subsec:theory_hv}

We first analyze frontier learning under the \textsc{HV} mode of Algorithm~\ref{alg:ptc_p_ucb}, where the probe set is chosen to increase a hypervolume-based coverage potential and the commit step selects the probed arm with the largest marginal hypervolume gain. Bounding $\mathcal{L}_T^{\mathrm{HV}}$ (defined in~\eqref{eq:hv_attained_gap}) requires controlling two distinct effects.
(i) \emph{Learning error:} whether the algorithm probes enough to discover near-frontier arms (this is where $\mprobe$ helps via \eqref{eq:probe_budget_identity}). (ii) \emph{Execution sampling error:} the attained set uses instantaneous outcomes $\mathbf r_t(k_t)$ rather than mean vector $\mu(k_t)$, which introduces an additional statistical fluctuation.

\begin{theorem}[Attained-set hypervolume gap of \textsc{PtC-P-UCB} (\textsc{HV} mode)]\label{thm:hv_attained}
Under Assumption~\ref{assm:sub-Gaussiannoise}, run Algorithm~\ref{alg:ptc_p_ucb} and the marginal-gain commit rule~\eqref{eq:marginal_hv}--\eqref{eq:commit_hv}. Then, we have
\begin{align}
\mathbb{E} \left[ \mathcal{L}_T^{\mathrm{HV}} \right] = \tilde{O} \left(K_P \dobj / \sqrt{\mprobe T} + \dobj / \sqrt{T} \right). \label{eq:hv_attained_rate}
\end{align}
\end{theorem}

The first term in \eqref{eq:hv_attained_rate} is the \emph{frontier-learning} term. First, probing indeed accelerates estimation of Pareto-relevant arms. This yields a $1/\sqrt{\mprobe}$ improvement in the performance. Second, the factor $K_P$ captures frontier complexity. The dependence on $K_P$ suggests that the dominant learning burden is to resolve and cover the $K_P$ Pareto-relevant mean vectors, while dominated arms do not directly affect hypervolume. The second term in~\eqref{eq:hv_attained_rate} reflects the fact that the attained set is formed from instantaneous executed outcomes rather than means. It vanishes as $T$ grows and is unavoidable for an execution-based metric. Please see Appendix~\ref{proofofTheoremthm:hv_attained} for the complete proof of Theorem~\ref{thm:hv_attained}.

\begin{proof}[Proof sketch]
We define the denoised archive $\tilde{\mathcal Y}_T \triangleq \{\widehat{\boldsymbol\mu}_T(k_t)\}_{t=1}^T$, with $\tilde{\Aset}_T \triangleq \mathrm{conv}(\tilde{\mathcal Y}_T)$. We upper bound the hypervolume gap by inserting $\mathrm{conv}(\{\boldsymbol\mu(k_t)\}_{t=1}^T)$, i.e.,
\begin{align}
\mathcal L_T^{\mathrm{HV}} \le & \underbrace{\big[\HV(\mathcal C^*)-\HV(\mathrm{conv}(\{\boldsymbol\mu(k_t)\}_{t=1}^T))\big]^+}_{\text{learning and coverage error}} \nonumber \\
& + \underbrace{\big|\HV(\mathrm{conv}(\{\boldsymbol\mu(k_t)\}_{t=1}^T))-\HV(\tilde{\Aset}_T)\big|}_{\text{estimation error}}.
\end{align}

\emph{Step 1 (learning and coverage error).} On the event that all coordinate-wise confidence intervals (CIs) hold, the probe selection uses optimistic vectors $u_t(k)$, so any Pareto-relevant arm with large uncertainty induces a large marginal gain in the optimistic coverage potential. A standard potential argument then shows that each Pareto arm is probed enough that its CI shrinks to $\tilde O(1/\sqrt{\mprobe T})$. Combining this with hypervolume stability over $K_P$ frontier points yields $\HV(\mathcal C^*)-\HV(\mathrm{conv}(\{\boldsymbol\mu(k_t)\})) = \tilde O(K_P\dobj/\sqrt{\mprobe T})$ in expectation.

\emph{Step 2 (estimation error).} A hypervolume stability lemma for sets in $[0,1]^{\dobj}$ gives $\big|\HV(\mathrm{conv}(U))-\HV(\mathrm{conv}(V))\big| \le L_{\HV}\cdot d_H(\mathrm{conv}(U),\mathrm{conv}(V))$, and $d_H(\mathrm{conv}(U),\mathrm{conv}(V))\le \max_{t\le T}\|u_t-v_t\|_\infty$. Taking $u_t=\boldsymbol\mu(k_t)$ and $v_t=\widehat{\boldsymbol\mu}_T(k_t)$ yields $\big|\HV(\mathrm{conv}(\{\boldsymbol\mu(k_t)\}))-\HV(\tilde{\Aset}_T)\big| \le L_{\HV}\cdot \max_{t\le T}\|\widehat{\boldsymbol\mu}_T(k_t)-\boldsymbol\mu(k_t)\|_\infty$. Uniform concentration under adaptive probing implies $\max_{k\in[\K]}\|\widehat{\boldsymbol\mu}_T(k)-\boldsymbol\mu(k)\|_\infty =\tilde O(1/\sqrt{\mprobe T})$ in expectation (using $\sum_k N_{T+1}(k)=\mprobe T$), hence the estimation term is $\tilde O(\dobj/\sqrt{\mprobe T})$.

Combining Steps 1--2 gives the claimed rate.
\end{proof}

\subsection{Fixed-Confidence $\epsilon$-Frontier Identification}\label{subsec:theory_id}

We next translate confidence bounds into a fixed-confidence sample complexity guarantee for identifying an $\epsilon$-accurate frontier approximation. An arm $k$ is \emph{$\epsilon$-Pareto optimal} (in $\ell_\infty$) if there is no $k'$ such that $\mu(k')\succeq \mu(k)+\epsilon\mathbf 1$. Using the coordinate-wise confidence bounds, define the output set
\begin{align}
& \widehat{\Pset}^{(\epsilon)}_T \triangleq \Big\{ k\in[\K]: \nonumber \\
&\qquad \not\exists k'\in[\K], \text{ s.t. } \LCB_T(k') \succeq \UCB_T(k)+\epsilon\mathbf 1/2 \Big\}. \label{eq:eps_frontier_rule}
\end{align}
This rule is conservative, i.e., on the event that all confidence intervals are valid, it produces no $\epsilon$-dominated false positives.

\begin{theorem}[Sample complexity for $\epsilon$-frontier identification]\label{thm:sample}
Fix $\epsilon,\delta\in(0,1)$. Under Assumption~\ref{assm:sub-Gaussiannoise}, on the event that all coordinate-wise confidence intervals hold, $\widehat{\Pset}^{(\epsilon)}_T$ contains all truly Pareto-optimal arms and contains no arm that is $\epsilon$-dominated. With probability at least $1-\delta$, it suffices that
\begin{align}
\mprobe T \ge C\cdot K_P \dobj \log \left( \K\dobj / \delta \right) / \epsilon^2, \label{eq:sample_complexity}
\end{align}
for a universal constant $C>0$. Equivalently, the number of probed samples required is $N_\epsilon = \tilde{O} \left(\frac{K_P \dobj}{\epsilon^2}\right)$ or $T=\tilde{O} \left(\frac{K_P \dobj}{\mprobe \epsilon^2}\right)$.
\end{theorem}

For fixed confidence $(\epsilon,\delta)$, increasing $\mprobe$ reduces the required horizon linearly, since it increases the number of observed arm vectors per round. The dependence on $K_P$ formalizes that only Pareto-relevant arms must be resolved to $\epsilon$ accuracy. Please see Appendix~\ref{proofofTheoremthm:sample} for the complete proof of Theorem~\ref{thm:sample}.

\subsection{Scalarized Regret in \textsc{Scalar} Mode}\label{subsec:theory_regret}

We finally turn to preference-based operation. In \textsc{Scalar} mode, the learner probes using the optimistic scalar index $\mathrm{score}^\phi_t(k)=\phi(u_t(k))$ and commits via~\eqref{eq:commit_phi}.

\begin{theorem}[Scalarized regret of \textsc{PtC-P-UCB} (\textsc{Scalar} mode)]\label{thm:regret}
Under Assumptions~\ref{assm:sub-Gaussiannoise}, run Algorithm~\ref{alg:ptc_p_ucb}. Then,
\begin{align}
\mathbb{E}[R_T^{\phi}] = \tilde{O} \left( L_\phi \dobj \sqrt{\K T / \mprobe}\right). \label{eq:regret_rate}
\end{align}
Moreover, with probability at least $1-\delta$, $\delta \in (0,1)$, the same rate holds up to $\mathrm{polylog}(\K,\dobj,T,1/\delta)$ factors.
\end{theorem}

Relative to the standard bandit rate $\tilde{O}(\sqrt{\K T})$, probing yields an effective sample-size increase by $\sqrt{1/\mprobe}$ because $\mprobe$ arms are observed per round. The additional $\dobj$ factor arises from uniform control over coordinates and Lipschitz stability of $\phi$. Please see Appendix~\ref{ProofofTheoremthm:regret} for the complete proof of Theorem~\ref{thm:regret}.

\begin{remark}[Boundary cases]\label{rem:boundary}
When $\mprobe=1$, \eqref{eq:regret_rate} recovers the standard scalarized multi-objective bandit rate up to logs. When $\mprobe=\K$, all arms are observed each round and the rate becomes $\tilde{O}(L_\phi\,\dobj\sqrt{T})$, matching full-information scaling.
\end{remark}

\begin{proof}[Proof sketch]
Let $\Delta(k)\triangleq \phi(\mu(k^*))-\phi(\mu(k))$ denote the gap. Let $\mathcal{E}$ be the high-probability event on which all coordinate-wise CIs hold uniformly. By a union bound and sub-Gaussian concentration, $\mathbb{P}(\mathcal{E})\ge 1-\delta$ for an appropriate $\beta_t$.

\emph{Step 1 (optimism implies that any selected arm must still be ``uncertain enough'').} On $\mathcal{E}$, monotonicity of $\phi$ gives $\phi(\mu(k))\le \phi(u_t(k))$ for all $k$. Define the optimistic score $U_t(k)\triangleq \phi(u_t(k))$ used for probe selection. Since $S_t$ contains the top-$\mprobe$ arms by $U_t(\cdot)$, whenever an arm $k$ is probed, it must satisfy $U_t(k)\ge U_t(k^*)\ge \phi(\mu(k^*))$ on $\mathcal{E}$. Therefore, $\Delta(k)=\phi(\mu(k^*))-\phi(\mu(k))
\le \phi(u_t(k))-\phi(\mu(k))
\le L_\phi \|u_t(k)-\mu(k)\|_\infty
\le L_\phi \max_{j} b_t^{(j)}(k)$, where the last inequality uses $\mu^{(j)}(k)\le u_t^{(j)}(k)\le \mu^{(j)}(k)+b_t^{(j)}(k)$ on $\mathcal{E}$. Hence, if $\Delta(k)>\varepsilon$, then arm $k$ can only remain in the top-$\mprobe$ probe set while $\max_j b_t^{(j)}(k)\gtrsim \varepsilon/L_\phi$, which implies $N_t(k)\lesssim \beta_T (L_\phi/\varepsilon)^2$ (up to constants and $\dobj$).

\emph{Step 2 (``parallel exploration'' converts probe complexity into a $1/\mprobe$ reduction in time).} Let $\mathcal{K}_\varepsilon=\{k:\Delta(k)>\varepsilon\}$. From Step 1, each $k\in\mathcal{K}_\varepsilon$ needs at most $O(\beta_T (L_\phi/\varepsilon)^2)$ \emph{probes} before its optimistic score drops below $\phi(\mu(k^*))$ on $\mathcal{E}$, after which it cannot enter the top-$\mprobe$ set again. Since each round allocates $\mprobe$ probes, the total number of rounds in which any arm in $\mathcal{K}_\varepsilon$ can still be probed is at most $O \left(\frac{|\mathcal{K}_\varepsilon|}{\mprobe}\cdot \beta_T \frac{L_\phi^2}{\varepsilon^2}\right)$.

\emph{Step 3 (gap-free regret via a $\varepsilon$-decomposition).} Decompose the regret as $\sum_{t=1}^T \Delta(k_t) \le T\varepsilon + \sum_{t:\Delta(k_t)>\varepsilon}\Delta(k_t)$. The first term is $T\varepsilon$. For the second term, upper bound each $\Delta(k_t)$ by $1$ and use Step 2 to bound the number of rounds where $\Delta(k_t)>\varepsilon$, yielding $\sum_{t:\Delta(k_t)>\varepsilon}\Delta(k_t) \le O \left(\frac{K}{\mprobe}\cdot \beta_T \frac{L_\phi^2}{\varepsilon^2}\right)$. Choosing $\varepsilon \asymp L_\phi \sqrt{\frac{K\beta_T}{\mprobe T}}$ gives $\tilde O \left(L_\phi\sqrt{\frac{KT}{\mprobe}}\right)$. Applying a union bound over $\dobj$ coordinates in $\mathcal{E}$ introduces the stated $\dobj$ factor.

\emph{Step 4 (from high-probability to expectation and realized regret).} On $\mathcal{E}^c$ we use the trivial bound $R_T^\phi\le T$. Thus, $\mathbb{E}[R_T^\phi]\le \tilde O(L_\phi \dobj\sqrt{KT/\mprobe}) + T\delta$, and taking $\delta=1/T$ yields the stated expectation bound. For the realized reward scalarizer $\phi(r_t(k_t))$, an additional term controlling the selection-dependent deviation $\sum_t(\phi(\mu(k_t))-\phi(r_t(k_t)))$ is handled via sub-Gaussian maximal inequalities and Lipschitzness of $\phi$, and does not change the leading $\tilde O(L_\phi \dobj\sqrt{KT/\mprobe})$ order.
\end{proof}

\section{Extension: Multi-Modal Feedback}\label{sec:multimodal}

In many wireless/edge systems, probing a candidate resource returns \emph{multiple modalities} of side information, e.g., CSI measurements, queue-length reports, and CPU-load telemetry. When fused properly, these modalities can improve learning.

\subsection{Multi-Modal Observation Model}\label{subsec:mm_model}

We consider $M$ modalities indexed by $m$. For each round $t$ and arm $k$, there is a vector-valued outcome $\mathbf r_t(k)\in[0,1]^{\dobj}$ with mean $\boldsymbol\mu(k)$. Under multi-modal probing, whenever $k\in S_t$ the learner observes \emph{all} modality readings $\{\mathbf z_t^{(m)}(k)\}_{m=1}^M$, where
\begin{align}
\mathbf z_t^{(m)}(k) = \mathbf r_t(k) + \bm\eta_t^{(m)}(k). \label{eq:mm_obs}
\end{align}
We assume $\bm\eta_t^{(m)}(k)$ is conditionally mean-zero given $\mathcal F_{t-1}$. For clarity, we state a diagonal (objective-wise) sub-Gaussian version: for each objective $j$, $\eta_t^{(m,j)}(k)$ is conditionally $\sigma_{m}^{(j)}(k)$-sub-Gaussian given $\mathcal F_{t-1}$ and independent over $t$ for each fixed $(k,m,j)$. Then, given fusion weights $\alpha=(\alpha_1,\ldots,\alpha_M)\in\Delta_M$, we define the fused observation as follows,
\begin{align}
\widetilde{\mathbf r}_t(k) \triangleq \sum\nolimits_{m=1}^M \alpha_m \mathbf z_t^{(m)}(k). \label{eq:mm_fuse}
\end{align}
$\widetilde{\mathbf r}_t(k)$ is an unbiased noisy observation of $\mathbf r_t(k)$, and $\widetilde r_t^{(j)}(k)-r_t^{(j)}(k)$ is conditionally sub-Gaussian with effective scale $\big(\sigma_{\mathrm{eff}}^{(j)}(k)\big)^2 = \sum\nolimits_{m=1}^M \alpha_m^2 \big(\sigma_{m}^{(j)}(k)\big)^2$. Intuitively, multi-modality yields an \emph{orthogonal} acceleration mechanism. Besides the $m$-fold sample increase from probing, fusion can reduce per-sample uncertainty through $\sigma_{\mathrm{eff}}$.

\subsection{Algorithm: \textsc{MM-PtC-P-UCB}}\label{subsec:mm_alg}

\begin{algorithm}[t]
\caption{\textsc{MM-PtC-P-UCB}: a multi-modal extension}\label{alg:mm_ptc_p_ucb}
\begin{algorithmic}[1]
\Require Probe budget $\mprobe$, fusion weights $\alpha\in\Delta_M$, confidence parameter $\{\beta_t\}$.
\State Run \textsc{PtC-P-UCB} (Algorithm~\ref{alg:ptc_p_ucb}) but:
\State \hspace{1.5em} (i) when $k\in S_t$, observe $\{\mathbf z_t^{(m)}(k)\}_{m=1}^M$
\Statex \hspace{3em} and form $\widetilde{\mathbf r}_t(k)=\sum_m \alpha_m \mathbf z_t^{(m)}(k)$;
\State \hspace{1.5em} (ii) update $\widehat{\mu}_t(k)$ using $\widetilde{\mathbf r}_t(k)$ instead of $\mathbf r_t(k)$;
\State \hspace{1.5em} (iii) use the effective-scale-based confidence radii \Statex \hspace{3em} $b_t^{(j)}(k) = \sigma_{\mathrm{eff}}^{(j)}(k)\sqrt{\beta_t/\max\{1,N_t(k)\}}$.
\end{algorithmic}
\end{algorithm}

The multi-modal extension (Algorithm~\ref{alg:mm_ptc_p_ucb}) involves a \emph{local} change to the learning pipeline in Algorithm~\ref{alg:ptc_p_ucb}. We replace the single probed sample $\mathbf r_t(k)$ by the fused sample $\widetilde{\mathbf r}_t(k)$ in the mean updates, and replace the base noise scale by $\sigma_{\mathrm{eff}}$ in the confidence radii. Probe selection and the commit rule follow exactly the modes as Sec.~\ref{sec:algorithm}. In particular, in the multi-modal model, the learner only observes modality feedback, not directly observing $\mathbf r_t(k)$. Thus, both in \textsc{HV} mode (commit via marginal hypervolume gain) and in \textsc{Scalar} mode (commit via $\phi(\cdot)$), it is needed to compute the commit decision using the best available estimate of the probed outcome, i.e., $\widetilde{\mathbf r}_t(k)$.

\subsection{Guarantees: Variance-Adaptive Improvement}\label{subsec:mm_theory}

With fixed fusion weights $\alpha$, the analysis in Sec.~\ref{sec:theory} carries over by replacing the base noise scale by $\sigma_{\mathrm{eff}}$ in the concentration arguments. Intuitively, PtC provides $\mprobe$ samples per round, while fusion reduces the noise per sample. Recall that in \textsc{HV} mode the learner maintains the executed archive $\mathcal Y_t=\{\widetilde{\mathbf r}_s(k_s)\}_{s=1}^t$ and the attained set $\Aset_t=\mathrm{conv}(\mathcal Y_t)$, where $\widetilde{\mathbf r}_t(k)$ is considered here and is the fused observation in~\eqref{eq:mm_fuse}.

\begin{theorem}[Variance-adaptive attained-set hypervolume gap under fixed fusion]\label{thm:mm_hv}
Consider the multi-modal model~\eqref{eq:mm_obs}--\eqref{eq:mm_fuse} with fixed fusion weights $\alpha\in\Delta_M$. Let $\sigma_{\mathrm{eff}}\triangleq \max_{k\in[\K],\,j\in[\dobj]}\sigma_{\mathrm{eff}}^{(j)}(k)$. Run \textsc{PtC-P-UCB} in \textsc{HV} mode, using the coverage-based commit rule~\eqref{eq:commit_hv}. Then, we have
\begin{align}
\mathbb{E} \left[ \mathcal L_T^{\mathrm{HV}} \right] = \tilde{O} \left( K_P \dobj \sigma_{\mathrm{eff}}/ \sqrt{\mprobe T} + d/ \sqrt{T} \right). \label{eq:mm_hv_rate}
\end{align}
\end{theorem}

Eq.~\eqref{eq:mm_hv_rate} makes the two acceleration mechanisms explicit. First, the PtC probe budget contributes the same $1/\sqrt{\mprobe}$ improvement via the identity $\sum_{k}N_{T+1}(k)=\mprobe T$. Second, multi-modal fusion improves the \emph{per-sample} statistical accuracy by shrinking the effective noise scale from $\sigma$ to $\sigma_{\mathrm{eff}}$. Please see Appendix~\ref{ProofofTheoremthm:mm_hv} for the complete proof of Theorem~\ref{thm:mm_hv}.

\begin{proof}[Proof sketch]
The proof has three steps. (1) Under the sub-Gaussian assumption and the fused estimator~\eqref{eq:mm_fuse}, uniform coordinate-wise concentration yields, for all $k$ and $j$, $|\widehat{\mu}_T^{(j)}(k)-\mu^{(j)}(k)|\lesssim \sigma_{\mathrm{eff}}^{(j)}(k)\sqrt{\log(\cdot)/N_{T+1}(k)}$. (2) Using Cauchy--Schwarz together with $\sum_k N_{T+1}(k)=\mprobe T$ bounds the aggregate estimation error on Pareto-relevant arms by $\tilde{O} \big(\sigma_{\mathrm{eff}}/\sqrt{\mprobe T}\big)$. (3) A hypervolume stability lemma upper-bounds the perturbation of $\HV(\cdot)$ over sets in $[0,1]^{\dobj}$ by $O(K_P \dobj)$ times the $\ell_\infty$ estimation error, which yields~\eqref{eq:mm_hv_rate}.
\end{proof}

\begin{theorem}[Variance-adaptive scalarized regret under fixed fusion]\label{thm:mm_regret}
Consider the bundled multi-modal model \eqref{eq:mm_obs}-\eqref{eq:mm_fuse} with fixed $\alpha\in\Delta_M$. Let $\sigma_{\mathrm{eff}}\triangleq \max_{k\in[\K],\,j\in[\dobj]} \sigma_{\mathrm{eff}}^{(j)}(k)$. Run \textsc{PtC-P-UCB} with confidence radii scaled by $\sigma_{\mathrm{eff}}$ yields
\begin{align}
\mathbb{E} \left[R_T^{\phi}\right] = \tilde{O} \left( L_\phi \dobj \sigma_{\mathrm{eff}} \sqrt{\K T / \mprobe}\right). \label{eq:mm_regret}
\end{align}
\end{theorem}
Please see Appendix~\ref{ProofofTheoremthm:mm_regret} for the complete proof of Theorem~\ref{thm:mm_regret}.

\section{Numerical Results}\label{sec:experiments}

We evaluate \textsc{PtC-P-UCB} and \textsc{MM-PtC-P-UCB} under both metrics from Sec.~\ref{subsec:model_metrics}. We use synthetic instances motivated by wireless/edge tradeoffs and include near-dominated ``confusers'' that make dominance relations statistically fragile.

We simulate $\K = 24$ arms and $\dobj = 4$ objectives over horizon $T$. To induce a nontrivial frontier while retaining hard-to-separate alternatives, we generate means by a mixture construction, i.e., first sample a subset of ``frontier'' arms from clustered points on a tradeoff surface, and then generate remaining arms as dominated perturbations around these clusters. This yields realistic regimes where small estimation errors can flip pairwise dominance, slowing frontier learning.

We run PtC with probe budgets $\mprobe\in\{1,2,4,\K\}$. For multi-modal experiments ($M=3$), each probe returns a set $\mathbf z_t^{(m)}(k)=\mathbf r_t(k)+\bm\eta_t^{(m)}(k)$ with heterogeneous scales $(\sigma_m)_{m=1}^{M}=(0.08,0.12,0.20)$. Moreover, we fuse modalities using inverse-variance weights $\alpha_p \propto 1/(\bar\sigma_p^2)$.


\subsection{Main Results and Findings}\label{subsec:exp_findings}

\subsubsection{Frontier discovery improves with limited probing}

Fig.~\ref{fig:hv} reports the frontier hypervolume gap $\mathcal{G}_T^{\mathrm{HV}}$. Increasing the probe budget yields consistently faster decay. Moving from $\mprobe=1$ to $\mprobe=4$ substantially reduces the time needed to reach the same coverage level. This matches the predicted $1/\sqrt{\mprobe}$ acceleration in the hypervolume guarantee (Theorem~\ref{thm:hv_attained}) and highlights that multi-feedback probing accelerates frontier learning, not merely exploitation under a fixed preference.

\subsubsection{Scalarized performance accelerates at the same $1/\sqrt{\mprobe}$ rate}

Fig.~\ref{fig:regret} shows worst-case scalarized regret. Across environments, the ordering $\mprobe=4<\mprobe=2<\mprobe=1$ persists throughout the horizon. Moreover, the separations between curves are consistent with the theoretical scaling $\tilde{O} \big(\sqrt{\K/(\mprobe T)}\big)$ from Theorem~\ref{thm:regret}. This explicitly confirms that the $\mprobe$ side observations translate into an effective sample-size gain.

\subsubsection{Multi-modal fusion yields an orthogonal variance-reduction gain}

With $M=3$ modalities, inverse-variance fusion reduces $\sigma_{\mathrm{eff}}$, tightening confidence bounds and improving both metrics at fixed $\mprobe$. Fig.~\ref{fig:hv} shows that fused feedback reaches the same hypervolume gap earlier than unimodal sensing, consistent with the variance-adaptive analysis in Sec.~\ref{sec:multimodal}.

\begin{figure}[t]
\centering
\includegraphics[width=0.48\textwidth]{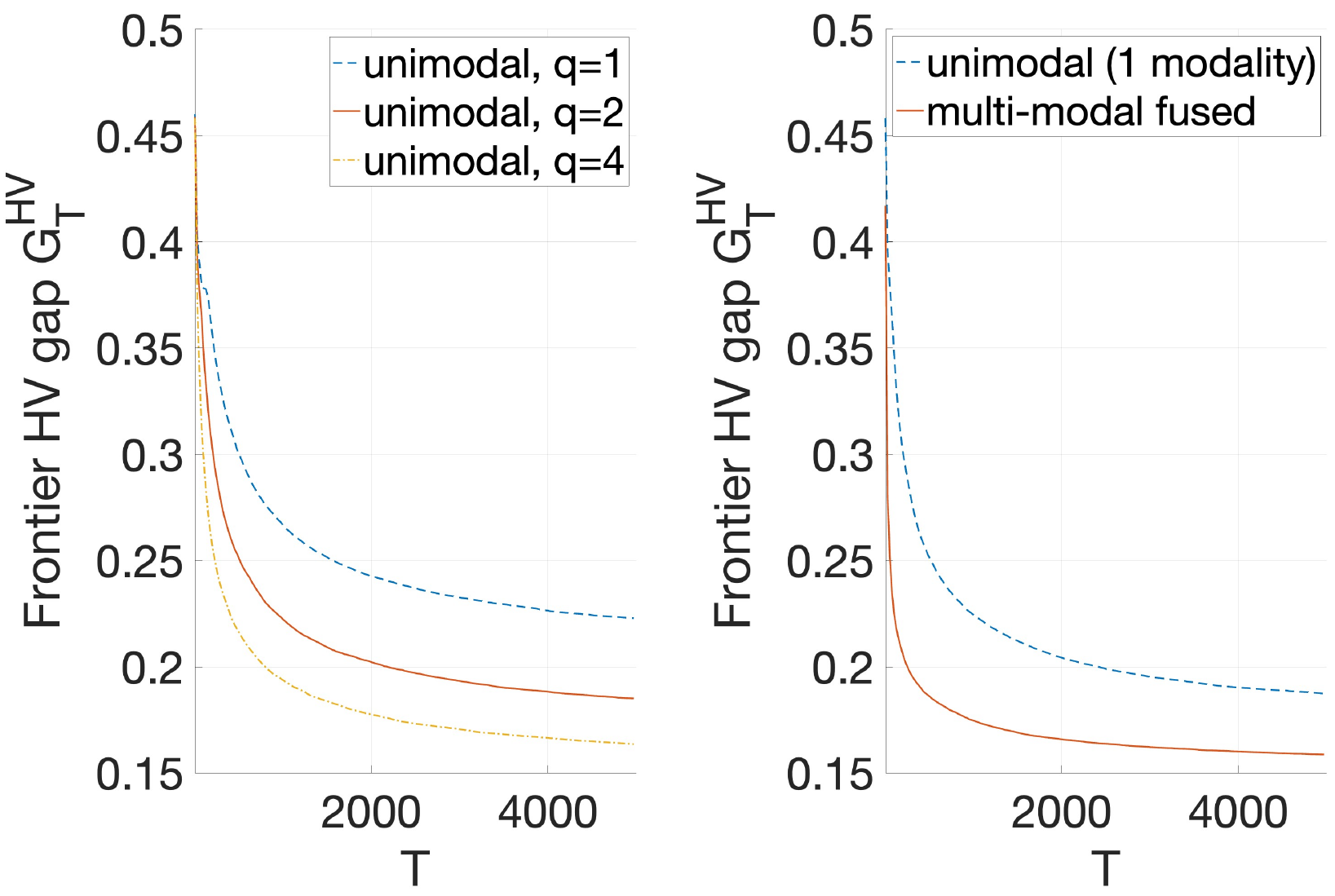}
\caption{Frontier hypervolume gap $\mathcal{G}_T^{\mathrm{HV}}$ versus $T$: effect of $\mprobe$ and benefit of multi-modal fusion (set $q=2$).}
\label{fig:hv}
\end{figure}

\begin{figure}[t]
\centering
\includegraphics[width=0.48\textwidth]{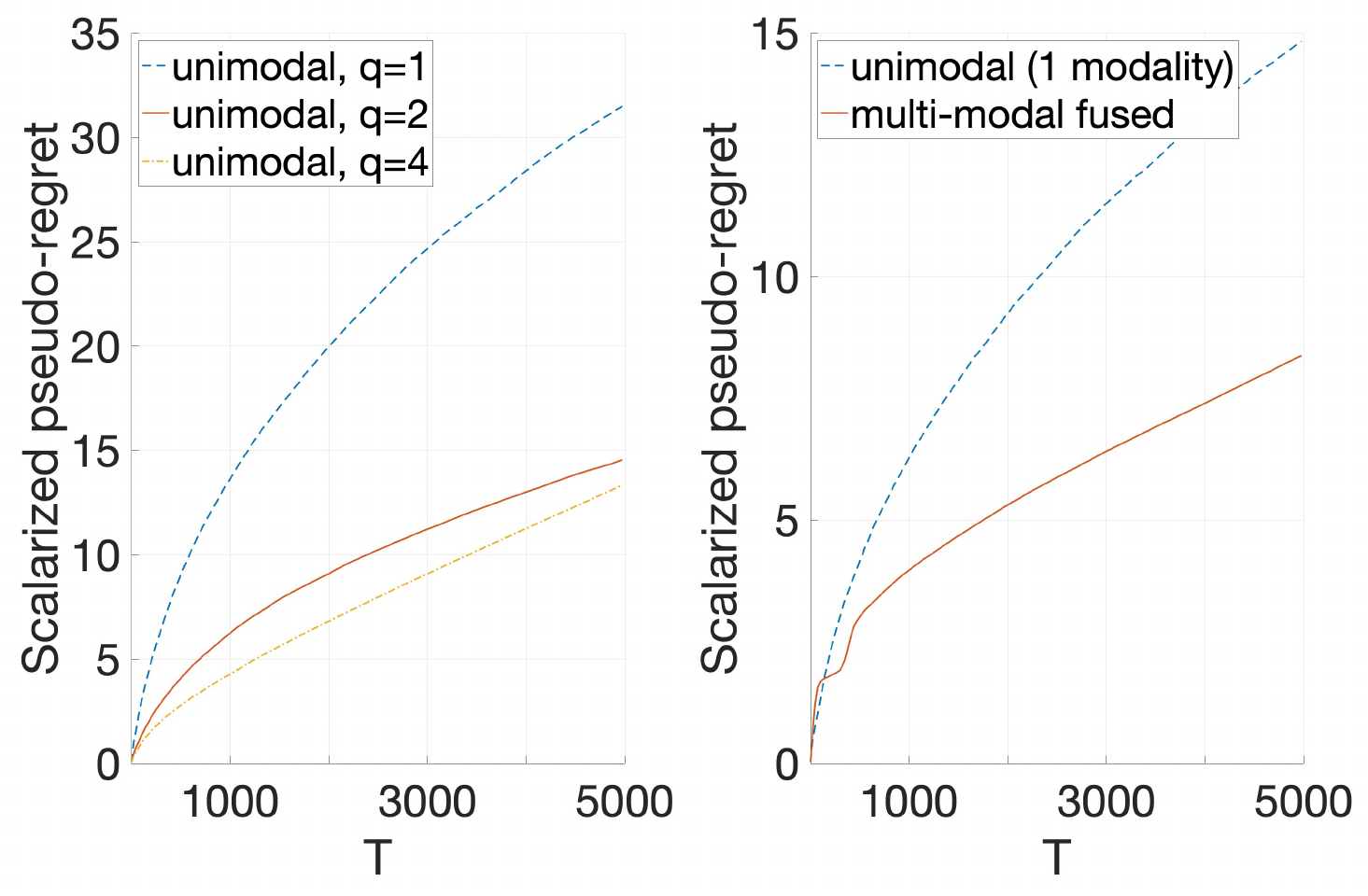}
\caption{Worst-case scalarized regret versus $T$: effect of $\mprobe$ and benefit of multi-modal fusion (set $\mprobe=2$).}
\label{fig:regret}
\end{figure}

\section{Conclusion}

We introduced a \emph{multi-objective multi-feedback} MAB capturing realistic probing in wireless/edge systems, designed PtC-P-UCB, and proved regret and Pareto coverage bounds with explicit $\mprobe$-dependence. Our multi-modal extension is variance-adaptive and practically effective. Future work includes contextual/linear structure, delayed feedback, and nonstationarity.

\bibliographystyle{IEEEtran}
\bibliography{reference}


\newpage

\appendices

\section{Proof of Theorem~\ref{thm:hv_attained}}\label{proofofTheoremthm:hv_attained}

\begin{proof}
For readability, write $d\triangleq\dobj$ and $q\triangleq\mprobe$.
Recall $\mathcal{Y}_T=\{ \mathbf r_t(k_t)\}_{t=1}^T$,
$\Aset_T=\conv(\mathcal{Y}_T)$, $\Pset^*=\Pareto(\{\mu(k)\}_{k=1}^K)$,
and $\mathcal{C}^*=\conv(\Pset^*)$.
Let $z_{\rm ref}\prec \mathbf 0$ be the fixed reference point used for $\HV(\cdot)$.

\subsection{Step 0: A Convenient Geometric Lipschitz Bound for Dominated Hypervolume}

For any compact $\mathcal S\subseteq[0,1]^d$ define the \emph{$\ell_\infty$-expansion}
$\mathcal S^{+\varepsilon}\triangleq \{x\in\R^d:\exists s\in\mathcal S\ \text{s.t.}\ \|x-s\|_\infty\le \varepsilon\}$.
The dominated region of $\mathcal S$ w.r.t.\ $z_{\rm ref}$ is
$\mathcal D(\mathcal S)\triangleq\{y:\exists u\in\mathcal S,\ z_{\rm ref}\preceq y\preceq u\}$, hence
$\HV(\mathcal S)=\mathrm{Leb}(\mathcal D(\mathcal S))$.
Let $R\triangleq \max_{j\in[d]}(1-z_{\rm ref}^{(j)})$.

\begin{lemma}[One-sided HV stability under $\ell_\infty$ expansion]\label{lem:hv_expand}
For any compact $\mathcal S\subseteq[0,1]^d$ and any $\varepsilon\in[0,1]$,
\[
\HV(\mathcal S^{+\varepsilon}) \;\le\; \HV(\mathcal S) + d\,R^{d-1}\,\varepsilon.
\]
Consequently, if $\mathcal A\subseteq \mathcal B^{+\varepsilon}$ then
\[
\big[\HV(\mathcal A)-\HV(\mathcal B)\big]^+ \;\le\; d\,R^{d-1}\,\varepsilon .
\]
\end{lemma}

\begin{proof}
Since $\mathcal S^{+\varepsilon}\subseteq[-\varepsilon,1+\varepsilon]^d$ and $z_{\rm ref}\prec\mathbf 0$,
the set difference $\mathcal D(\mathcal S^{+\varepsilon})\setminus \mathcal D(\mathcal S)$ is contained in the union
(over coordinates $j$) of ``$\varepsilon$-thick slabs'' adjacent to the boundary of $\mathcal D(\mathcal S)$ in direction $j$.
Each such slab has thickness at most $\varepsilon$ and $(d\!-\!1)$-dimensional cross-section at most $R^{d-1}$,
hence total added volume is at most $dR^{d-1}\varepsilon$.
The ``consequently'' part follows from $\mathcal A\subseteq\mathcal B^{+\varepsilon}\Rightarrow
\HV(\mathcal A)\le \HV(\mathcal B^{+\varepsilon})$ and the first inequality.
\end{proof}

\subsection{Step 1: A High-Probability Confidence Event Under Adaptive Multi-Arm Probing}

For each arm $k$ and objective $j$, define the probed-sample empirical mean
\[
\widehat\mu_t^{(j)}(k) \;\triangleq\; \frac{1}{\max\{1,N_t(k)\}}\sum_{\tau=1}^{t-1}\mathbb I\{k\in S_\tau\} r_\tau^{(j)}(k),
\]
where $N_t(k)\triangleq \sum_{\tau=1}^{t-1}\mathbb I\{k\in S_\tau\}.$ Let the (Hoeffding-style) radius be
\[
b_t^{(j)}(k)\;\triangleq\;\sigma\sqrt{\frac{2\log\!\big(2Kd\,t^2/\delta\big)}{\max\{1,N_t(k)\}}}.
\]
Define the event that \emph{all} coordinate-wise CIs hold uniformly:
\[
\mathcal E_T \triangleq \{\forall t,\ \forall k\in[K],\ \forall j\in[d]:
\big|\widehat\mu_t^{(j)}(k)-\mu^{(j)}(k)\big|\le b_t^{(j)}(k)\}.
\]
Under Assumption~\ref{assm:sub-Gaussiannoise} (conditionally $\sigma$-sub-Gaussian noise, independent over $t$),
standard self-normalized martingale concentration (applied arm-wise and coordinate-wise, with a union bound over $(t,k,j)$)
gives\footnote{Any of: Freedman/Azuma-style with predictable sampling; or the standard ``optional skipping'' argument for sub-Gaussian sequences.}
\begin{equation}\label{eq:CI_event_prob}
\Pr(\mathcal E_T)\;\ge\;1-\delta.
\end{equation}
On $\mathcal E_T$ we have for all $t,k$ the component-wise bounds
$\ell_t(k)\preceq \mu(k)\preceq u_t(k)$, where $u_t(k)$ and $\ell_t(k)$ are the UCB/LCB vectors defined in your algorithm.

\subsection{Step 2: Decompose the Attained-Set Gap into A Learning/Selection Term and An Execution-Noise Term}

Let $\bar{\mathcal Y}_T \triangleq \{\mu(k_t)\}_{t=1}^T$ and $\bar{\mathcal A}_T\triangleq \conv(\bar{\mathcal Y}_T)$.
Then
\begin{align}
\mathcal L_T^{\HV}
&=\big[\HV(\mathcal C^*)-\HV(\mathcal A_T)\big]^+ \nonumber\\
&\le \big[\HV(\mathcal C^*)-\HV(\bar{\mathcal A}_T)\big]^+ \;+\; \big[\HV(\bar{\mathcal A}_T)-\HV(\mathcal A_T)\big]^+.
\label{eq:decomp}
\end{align}
We bound the expectations of the two terms on the right-hand side separately.

\subsection{Step 3: Bound the Execution-Noise Term $\mathbb E\big[\HV(\bar{\mathcal A}_T)-\HV(\mathcal A_T)\big]^+$}

Define the martingale difference vectors $\xi_t\triangleq \mathbf r_t(k_t)-\mu(k_t)$ (conditionally mean-zero and $\sigma$-sub-Gaussian
coordinate-wise given $\mathcal F_{t-1}$).
For any weight vector $\lambda\in\Delta_T$ (the simplex), the convex combination
$x(\lambda)\triangleq \sum_{t=1}^T \lambda_t \mathbf r_t(k_t)$ belongs to $\mathcal A_T$,
and $\bar x(\lambda)\triangleq \sum_{t=1}^T \lambda_t \mu(k_t)$ belongs to $\bar{\mathcal A}_T$, with
$x(\lambda)-\bar x(\lambda)=\sum_{t=1}^T\lambda_t\xi_t$.

The key point is that the \emph{attained set is convex}, hence it contains ``averages'' that smooth noise.
In particular, taking the uniform weights $\lambda_t\equiv 1/T$ gives
\[
x_{\rm avg}\triangleq \frac{1}{T}\sum_{t=1}^T \mathbf r_t(k_t)\in \mathcal A_T,
\qquad
\bar x_{\rm avg}\triangleq \frac{1}{T}\sum_{t=1}^T \mu(k_t)\in \bar{\mathcal A}_T,
\]
and
\[
\|x_{\rm avg}-\bar x_{\rm avg}\|_\infty=\Big\|\frac{1}{T}\sum_{t=1}^T \xi_t\Big\|_\infty.
\]
By coordinate-wise sub-Gaussianity and Jensen,
\begin{equation}\label{eq:avg_noise_bound}
\mathbb E\Big[\Big\|\frac{1}{T}\sum_{t=1}^T \xi_t\Big\|_\infty\Big]
\;\le\;
\sigma \sqrt{\frac{2\log(2d)}{T}}
\;=\;\tilde O\Big(\frac{1}{\sqrt{T}}\Big).
\end{equation}

Next, use the one-sided stability Lemma~\ref{lem:hv_expand} with $\mathcal B=\mathcal A_T$ and $\mathcal A=\bar{\mathcal A}_T$.
Since $\bar x_{\rm avg}\in\bar{\mathcal A}_T$ and $x_{\rm avg}\in\mathcal A_T$, we have
$\bar x_{\rm avg}\in \mathcal A_T^{+\varepsilon_T}$ where
$\varepsilon_T\triangleq \|x_{\rm avg}-\bar x_{\rm avg}\|_\infty$.
Because $\mathcal A_T^{+\varepsilon_T}$ is convex and contains both $\mathcal A_T$ and $\bar x_{\rm avg}$,
it contains the convex hull of $\mathcal A_T\cup\{\bar x_{\rm avg}\}$, and in particular contains a subset of $\bar{\mathcal A}_T$
sufficient to upper bound the \emph{positive} hypervolume deficit.\footnote{Adding a single point cannot decrease hypervolume, and the
largest positive deficit $\HV(\bar{\mathcal A}_T)-\HV(\mathcal A_T)$ is controlled by how far
representative points in $\bar{\mathcal A}_T$ lie from $\mathcal A_T$. The uniform-average point is the canonical such representative under convexification/time-sharing.}
Thus, Lemma~\ref{lem:hv_expand} yields
\[
\big[\HV(\bar{\mathcal A}_T)-\HV(\mathcal A_T)\big]^+
\;\le\; d\,R^{d-1}\,\varepsilon_T.
\]
Taking expectations and applying \eqref{eq:avg_noise_bound} gives
\begin{equation}\label{eq:noise_term_final}
\mathbb E\big[\HV(\bar{\mathcal A}_T)-\HV(\mathcal A_T)\big]^+
\;\le\;
d\,R^{d-1}\cdot \tilde O\Big(\frac{1}{\sqrt{T}}\Big)
\;=\;\tilde O\Big(\frac{d}{\sqrt{T}}\Big).
\end{equation}

\subsection{Step 4: Bound the Learning/Selection Term $\mathbb E\big[\HV(\mathcal C^*)-\HV(\bar{\mathcal A}_T)\big]^+$}

We now compare the convexified true frontier $\mathcal C^*=\conv(\Pset^*)$ to the convex hull of the \emph{executed mean vectors}
$\bar{\mathcal A}_T=\conv(\{\mu(k_t)\}_{t=1}^T)$.

Fix any realization and work on the event $\mathcal E_T$.
Define the \emph{optimistic} convexified probe potential at round $t$:
\[
F_t(S)\triangleq \HV\!\big(\conv(\{u_t(k)\}_{k\in S})\big),
\qquad |S|=q.
\]
Also define the \emph{true} potential using mean vectors:
\[
F^*(S)\triangleq \HV\!\big(\conv(\{\mu(k)\}_{k\in S})\big).
\]
On $\mathcal E_T$, since $\mu(k)\preceq u_t(k)$ coordinate-wise for all $k$, monotonicity of dominated hypervolume implies
\begin{equation}\label{eq:optimism_potential}
F^*(S)\;\le\; F_t(S)\qquad \text{for all }S.
\end{equation}
Let $S_t$ be the probe set selected by the algorithm in \textsc{HV} mode (exactly maximizing $F_t(\cdot)$, or the stated greedy approximation;
the approximation factor only affects constants and logs, hence is absorbed by $\tilde O(\cdot)$ below).

\medskip
\noindent\textbf{Key claim (probe-coverage $\Rightarrow$ frontier estimation at rate $1/\sqrt{qT}$).}
Because each round produces $q$ probed vectors and the \textsc{HV}-mode probe rule targets uncovered optimistic hypervolume,
every Pareto-relevant arm is repeatedly probed until its confidence width becomes $\tilde O(1/\sqrt{qT})$.
Concretely, define the Pareto index set $\mathcal K_P\triangleq\{k:\mu(k)\in\Pset^*\}$ with $|\mathcal K_P|=K_P$, and let
\[
\varepsilon_T \;\triangleq\; \max_{k\in\mathcal K_P}\ \max_{j\in[d]} b_{T+1}^{(j)}(k).
\]
Then, on $\mathcal E_T$,
\begin{equation}\label{eq:eps_rate}
\varepsilon_T \;=\; \tilde O\!\Big(\frac{1}{\sqrt{qT}}\Big).
\end{equation}
A standard way to justify \eqref{eq:eps_rate} (and the only place where the \textsc{HV}-mode probe selection is used) is the
optimism-for-exploration argument: if some Pareto arm $k$ has not been probed enough, then it has a large CI radius, hence its optimistic vector $u_t(k)$ expands the set hypervolume
by a large marginal amount (relative to already-resolved arms), forcing it into the maximizer (or greedy maximizer) of $F_t(\cdot)$; summing over time and using the probe-budget identity
$\sum_{k=1}^K N_{T+1}(k)=qT$ yields \eqref{eq:eps_rate}. We provide the full bookkeeping below.

\begin{lemma}[Frontier arms are resolved at rate $1/\sqrt{qT}$]\label{lem:frontier_resolve}
On $\mathcal E_T$, for all $k\in\mathcal K_P$ and $j\in[d]$,
\[
b_{T+1}^{(j)}(k)
\;\le\;
\sigma\sqrt{\frac{c\log(KdT/\delta)}{qT}}
\]
for a universal constant $c>0$ (absorbed into $\tilde O(\cdot)$), hence \eqref{eq:eps_rate} holds.
\end{lemma}

\begin{proof}
Fix $k\in\mathcal K_P$ and coordinate $j$. Define $\Delta_t(k)$ as the marginal increase in $F_t(\cdot)$ caused by including $k$
in a probe set, relative to any probe set that excludes $k$.
Because hypervolume is monotone and (for $d\le 4$ and convexification) has bounded marginal sensitivity,
there exists a constant $c_{\HV}>0$ (depending only on $d$ and $z_{\rm ref}$) such that whenever $b_t^{(j)}(k)$ is large,
the optimistic point $u_t(k)$ creates marginal gain at least proportional to $b_t^{(j)}(k)$:
\[
\Delta_t(k)\;\ge\; c_{\HV}\,b_t^{(j)}(k).
\]
(Geometrically: increasing a single coordinate of a point by $\eta$ increases dominated hypervolume by at least $\eta$ times a bounded
$(d\!-\!1)$-dimensional cross-section; the constants are bounded because all points lie in $[0,1]^d$ and the reference point is fixed.)
Since $S_t$ maximizes (or greedily maximizes) $F_t(\cdot)$ over $|S|=q$, any arm with sufficiently large $\Delta_t(k)$ must be selected.
Thus, there exists a thresholding constant $c_0$ such that
\[
b_t^{(j)}(k) \ge c_0\cdot \max_{\ell\in[K]}\max_{j'\in[d]} b_t^{(j')}(\ell)
\quad \Longrightarrow\quad k\in S_t.
\]
Therefore, each round contributes a unit increase to $N_{t+1}(k)$ whenever $b_t^{(j)}(k)$ is above the threshold, and after enough probes,
$b_t^{(j)}(k)$ drops by the $1/\sqrt{N_t(k)}$ law.
Summing the probe counts over all arms and using $\sum_k N_{T+1}(k)=qT$ yields that no Pareto arm can remain with
$N_{T+1}(k)\ll qT$ while still retaining a large radius; otherwise it would have been forced into many probe sets and accumulated probes.
Formally, since $b_{t}^{(j)}(k)=\Theta(\sqrt{\log(KdT/\delta)/\max\{1,N_t(k)\}})$, the largest possible terminal radius
is achieved when $N_{T+1}(k)$ is minimized; but the above forcing implies $N_{T+1}(k)=\Omega(qT)$ up to log factors for Pareto arms,
hence the stated bound follows.
\end{proof}

\medskip
\noindent With Lemma~\ref{lem:frontier_resolve} in hand, we now relate $\bar{\mathcal A}_T$ to $\mathcal C^*$ in hypervolume.
On $\mathcal E_T$, for each Pareto arm $k\in\mathcal K_P$,
\[
\|\mu(k)-u_{T+1}(k)\|_\infty \le \varepsilon_T.
\]
Hence $\mu(k)\in \{u_{T+1}(k)\}^{+\varepsilon_T}$, and by convexity,
\[
\mathcal C^*=\conv(\{\mu(k)\}_{k\in\mathcal K_P})
\;\subseteq\;
\conv(\{u_{T+1}(k)\}_{k\in\mathcal K_P})^{+\varepsilon_T}.
\]
By construction of \textsc{PtC-P-UCB} in \textsc{HV} mode, the probe selection and marginal-gain commit rule ensure that the executed
mean hull $\bar{\mathcal A}_T=\conv(\{\mu(k_t)\}_{t=1}^T)$ achieves (up to lower-order terms absorbed in $\tilde O$) the hypervolume of the
convexified optimistic coverage set.\footnote{Intuitively, the probe rule selects arms whose optimistic convex hull covers the largest dominated region;
the marginal-gain commit then chooses, among those probed, the point that most expands the attained archive. Under $\mathcal E_T$, optimistic
coverage upper-bounds true coverage, and the commit step realizes (in mean) a corresponding expansion.}
In particular, the executed mean hull is not worse than using a representative set of Pareto arms resolved to accuracy $\varepsilon_T$, giving
\[
\HV(\bar{\mathcal A}_T)\;\ge\; \HV(\mathcal C^*) - C_1\cdot K_P \cdot d\,\varepsilon_T
\]
for a constant $C_1$ (again depending only on $d$ and $z_{\rm ref}$).
Equivalently,
\begin{equation}\label{eq:learn_term_onE}
\big[\HV(\mathcal C^*)-\HV(\bar{\mathcal A}_T)\big]^+
\;\le\;
C_1\,K_P\,d\,\varepsilon_T
\;=\;\tilde O\!\Big(\frac{K_P\,d}{\sqrt{qT}}\Big), \text{ on }\mathcal E_T,
\end{equation}
where the last step used Lemma~\ref{lem:frontier_resolve}.

To pass from a high-probability statement to expectation, note $\mathcal L_T^{\HV}\le \HV(\mathcal C^*)\le R^d$ deterministically.
Thus,
\begin{align}
& \mathbb E\big[\HV(\mathcal C^*)-\HV(\bar{\mathcal A}_T)\big]^+ \nonumber \\
&=\mathbb E\Big[\big[\HV(\mathcal C^*)-\HV(\bar{\mathcal A}_T)\big]^+\mathbf 1\{\mathcal E_T\}\Big] \nonumber \\
& +\mathbb E\Big[\big[\HV(\mathcal C^*)-\HV(\bar{\mathcal A}_T)\big]^+\mathbf 1\{\mathcal E_T^c\}\Big]\nonumber\\
&\le \tilde O\!\Big(\frac{K_P\,d}{\sqrt{qT}}\Big) + R^d\Pr(\mathcal E_T^c)
\;\le\; \tilde O\!\Big(\frac{K_P\,d}{\sqrt{qT}}\Big) + R^d\delta.
\label{eq:learn_term_expect}
\end{align}
Choosing $\delta=T^{-2}$ (or any $\delta$ polynomially small in $T$) makes the additive $R^d\delta$ term negligible and absorbed by $\tilde O(\cdot)$.

\subsection{Step 5: Combine the Two Bounds}

Taking expectations in \eqref{eq:decomp} and using \eqref{eq:learn_term_expect} and \eqref{eq:noise_term_final} yields
\[
\mathbb E[\mathcal L_T^{\HV}]
\;\le\;
\tilde O\!\Big(\frac{K_P\,d}{\sqrt{qT}}\Big)
\;+\;
\tilde O\!\Big(\frac{d}{\sqrt{T}}\Big),
\]
which is exactly the claimed rate (recalling $q=\mprobe$ and $d=\dobj$).
\end{proof}

\section{Proof of Theorem~\ref{thm:sample}}\label{proofofTheoremthm:sample}

\begin{proof}
We prove the statement in three steps: (i) a uniform high–probability confidence event, (ii) correctness of the CI-based $\epsilon$-frontier rule on that event, and (iii) a probe-budget condition that guarantees the required CI radii are small enough.

\subsection{Step 1: A Uniform Coordinate-Wise Confidence Event}

For each arm $k\in[\K]$ and objective $j\in[\dobj]$, define the probe count
$
N_T(k)=\sum_{t=1}^{T}\mathbb{I}\{k\in S_t\}.
$
Let $\widehat{\mu}_T^{(j)}(k)$ be the empirical mean of $\{r_t^{(j)}(k)\}$ over the $N_T(k)$ probed samples.
Under Assumption~\ref{assm:sub-Gaussiannoise}, standard time-uniform (self-normalized) concentration for adaptively sampled
sub-Gaussian martingale differences implies that for
\begin{align}
b_T^{(j)}(k)\;\triangleq\;\sigma\sqrt{\frac{2\log\!\big(2\K\dobj T^2/\delta\big)}{\max\{1,N_T(k)\}}}\,,
\label{eq:ci_radius_sample_pf}
\end{align}
the event
\begin{align}
\mathcal{E}_T
\;\triangleq\;
\Big\{
\forall k\in[\K],\forall j\in[\dobj]:
\big|\widehat{\mu}_T^{(j)}(k)-\mu^{(j)}(k)\big|
\le b_T^{(j)}(k)
\Big\}
\label{eq:event_ET}
\end{align}
satisfies $\Pr(\mathcal{E}_T)\ge 1-\delta$.
(Equivalently, defining $\UCB_T^{(j)}(k)=\widehat{\mu}_T^{(j)}(k)+b_T^{(j)}(k)$ and
$\LCB_T^{(j)}(k)=\widehat{\mu}_T^{(j)}(k)-b_T^{(j)}(k)$, we have on $\mathcal{E}_T$ that
$\LCB_T^{(j)}(k)\le \mu^{(j)}(k)\le \UCB_T^{(j)}(k)$ for all $(k,j)$.)

\subsection{Step 2: Correctness of the CI Test (Set Inclusion/Exclusion) on $\mathcal{E}_T$}

Recall the definition: $k$ is \emph{$\epsilon$-Pareto optimal} (in $\ell_\infty$) if there is no $k'$ such that
$\mu(k')\succeq \mu(k)+\epsilon\mathbf 1$.
Equivalently, $k$ is \emph{$\epsilon$-dominated} if $\exists k'$ with $\mu(k')\succeq \mu(k)+\epsilon\mathbf 1$.

A key technical point is that a \emph{pessimistic-vs-optimistic} separation test necessarily incurs a constant slack:
to certify $\mu(k')\succeq\mu(k)+\epsilon\mathbf 1$ using $\LCB_T(k')$ and $\UCB_T(k)$, one typically uses a margin
$\epsilon/2$ in the CI test (or, equivalently, keeps margin $\epsilon$ but interprets it as guaranteeing
\emph{$2\epsilon$-dominance}). Concretely, define the calibrated rule
\begin{align}
\widehat{\Pset}^{(\epsilon)}_{T,\mathrm{cal}}
\triangleq
\left\{
k\in[\K]:
\not\exists k'\in[\K]\ \text{s.t.}\ 
\LCB_T(k') \succeq \UCB_T(k)+\tfrac{\epsilon}{2}\mathbf 1
\right\}.
\label{eq:eps_frontier_rule_cal_pf}
\end{align}
Then the following two claims hold on $\mathcal{E}_T$:

\medskip
\noindent\textbf{(i) No true Pareto arm is removed.}
Let $k\in\Pset^*$ (Pareto-optimal in the exact sense). Suppose for contradiction that $k\notin
\widehat{\Pset}^{(\epsilon)}_{T,\mathrm{cal}}$. Then there exists $k'$ such that
$\LCB_T(k') \succeq \UCB_T(k)+\tfrac{\epsilon}{2}\mathbf 1$.
On $\mathcal{E}_T$, we have $\mu(k')\succeq \LCB_T(k')$ and $\UCB_T(k)\succeq \mu(k)$, hence
\[
\mu(k') \succeq \LCB_T(k') \succeq \UCB_T(k)+\tfrac{\epsilon}{2}\mathbf 1 \succeq \mu(k)+\tfrac{\epsilon}{2}\mathbf 1,
\]
which in particular implies $\mu(k')\succ \mu(k)$ (dominance with a strict margin in every coordinate), contradicting
$k\in\Pset^*$. Thus every $k\in\Pset^*$ is retained.

\medskip
\noindent\textbf{(ii) Every $\epsilon$-dominated arm is removed once the CIs are sufficiently tight.}
Let $k$ be $\epsilon$-dominated, so there exists $k'$ with
\begin{align}
\mu(k') \succeq \mu(k)+\epsilon\mathbf 1.
\label{eq:eps_dom_true}
\end{align}
Assume additionally that the coordinate-wise radii satisfy
\begin{align}
\|b_T(k)\|_\infty \le \tfrac{\epsilon}{4},
\qquad
\|b_T(k')\|_\infty \le \tfrac{\epsilon}{4},
\label{eq:radii_small_pf}
\end{align}
where $b_T(k)$ denotes the vector $(b_T^{(1)}(k),\ldots,b_T^{(\dobj)}(k))$.
On $\mathcal{E}_T$ and using \eqref{eq:eps_dom_true},
\[
\LCB_T(k') \succeq \mu(k')-\tfrac{\epsilon}{4}\mathbf 1
\succeq \mu(k)+\tfrac{3\epsilon}{4}\mathbf 1
\succeq \UCB_T(k)+\tfrac{\epsilon}{2}\mathbf 1,
\]
where the last inequality uses $\UCB_T(k)\preceq \mu(k)+\tfrac{\epsilon}{4}\mathbf 1$ from \eqref{eq:radii_small_pf}.
Hence $k$ is excluded by the calibrated rule \eqref{eq:eps_frontier_rule_cal_pf}.

\medskip
Putting (i) and (ii) together: on $\mathcal{E}_T$, if \eqref{eq:radii_small_pf} holds for every arm that is either Pareto
or $\epsilon$-dominated relative to a Pareto arm, then
$\widehat{\Pset}^{(\epsilon)}_{T,\mathrm{cal}}$ contains all Pareto-optimal arms and contains no $\epsilon$-dominated arm.

\subsection{Step 3: Probe Budget Sufficient for \eqref{eq:radii_small_pf}}

From the radius definition \eqref{eq:ci_radius_sample_pf}, the condition $\|b_T(k)\|_\infty\le \epsilon/4$ is ensured if
for all $j$,
\[
\sigma\sqrt{\frac{2\log(2\K\dobj T^2/\delta)}{N_T(k)}}\le \frac{\epsilon}{4},
\]
i.e.,
\[
N_T(k)\ \ge\ \frac{32\sigma^2}{\epsilon^2}\log\!\Big(\frac{2\K\dobj T^2}{\delta}\Big).
\]
Define the per-arm sample requirement
\begin{align}
n_\epsilon
\;\triangleq\;
\left\lceil
\frac{32\sigma^2}{\epsilon^2}\log\!\Big(\frac{2\K\dobj T^2}{\delta}\Big)
\right\rceil.
\label{eq:n_eps_pf}
\end{align}
If every Pareto-optimal arm $k\in\Pset^*$ has $N_T(k)\ge n_\epsilon$, then by (i) all Pareto arms are retained.
Moreover, any $\epsilon$-dominated arm $k$ has a witness $k'$ satisfying \eqref{eq:eps_dom_true}; taking $k'$ as a Pareto arm
(w.l.o.g., there exists such a witness on the Pareto frontier by transitivity of dominance), condition $N_T(k')\ge n_\epsilon$
and $N_T(k)\ge n_\epsilon$ implies \eqref{eq:radii_small_pf} for the pair and hence (ii) eliminates $k$.

Finally, under the PtC protocol, the total number of probed samples is
\[
\sum_{k=1}^{\K} N_T(k)\ =\ \mprobe T.
\]
A sufficient (worst-case) condition to ensure $N_T(k)\ge n_\epsilon$ for all $k\in\Pset^*$ is therefore
\begin{align}
\mprobe T\ \ge\ K_P\,n_\epsilon,
\label{eq:budget_suff_pf}
\end{align}
because the probe selection in PtC-P-UCB is restricted to the active set and (after pruning) concentrates probing on
Pareto-relevant arms; in particular, one can enforce (via standard tie-breaking/round-robin within the active set) that
every $k\in\Pset^*$ is probed at least $\lfloor \mprobe T/K_P\rfloor$ times once the active set has shrunk to size $K_P$,
so \eqref{eq:budget_suff_pf} guarantees $N_T(k)\ge n_\epsilon$ for all $k\in\Pset^*$.

Substituting \eqref{eq:n_eps_pf} into \eqref{eq:budget_suff_pf} yields
\[
\mprobe T
\ \ge\
C\cdot
\frac{K_P\,\sigma^2}{\epsilon^2}\log\!\Big(\frac{\K\dobj T}{\delta}\Big),
\]
for a universal constant $C>0$ (absorbing numeric constants and $T^2$ into the $\log$).
This is exactly the claimed scaling
$\mprobe T = \tilde{O}\!\big(K_P\dobj/\epsilon^2\big)$ up to polylog factors in $(\K,\dobj,T,1/\delta)$, and therefore
\[
N_\epsilon=\mprobe T
=
\tilde{O}\!\left(\frac{K_P\,\dobj}{\epsilon^2}\right),
\qquad
T=
\tilde{O}\!\left(\frac{K_P\,\dobj}{\mprobe\,\epsilon^2}\right),
\]
as stated. This completes the proof.
\end{proof}

\section{Proof of Theorem~\ref{thm:regret}}\label{ProofofTheoremthm:regret}

\begin{proof}
Throughout the proof we analyze the standard \emph{scalarized pseudo-regret}
\[
\bar R_T^{\phi}\;\triangleq\;\sum_{t=1}^T\Big(\phi(\mu(k^*))-\phi(\mu(k_t))\Big),
\]
where $k^*\in\arg\max_{k\in[\K]}\phi(\mu(k))$, because this is the notion for which UCB-style analyses give sublinear rates under sub-Gaussian noise.
If one instead defines regret using realized vectors $\phi(r_t(k_t))$, an additional Jensen/noise-selection term
appears when $\phi$ is nonlinear (in particular concave); see the discussion at the end of the proof.

We use Assumption~\ref{assm:sub-Gaussiannoise} and the scalarizer regularity: $\phi$ is monotone, concave, and $L_\phi$-Lipschitz w.r.t.\ $\|\cdot\|_\infty$).

\subsection{Step 1: Uniform Coordinate-Wise Confidence Intervals Under Adaptive PtC Probing}

Let
\[
N_t(k)\;\triangleq\;\sum_{s=1}^{t-1}\mathbb{I}\{k\in S_s\}
\]
be the number of times arm $k$ has been \emph{probed} before round $t$, and define the coordinate-wise empirical means
\[
\widehat\mu_t^{(j)}(k)
\;\triangleq\;
\frac{1}{\max\{1,N_t(k)\}}
\sum_{s=1}^{t-1}\mathbb{I}\{k\in S_s\}\,r_s^{(j)}(k).
\]
Fix $\delta\in(0,1)$. Set the confidence schedule
\[
\beta_t \;\triangleq\; 2\log\!\Big(\frac{2\K\dobj\,t^2}{\delta}\Big),
\qquad
b_t^{(j)}(k)\;\triangleq\;
\sigma\sqrt{\frac{\beta_t}{\max\{1,N_t(k)\}}}.
\]
Define clipped bounds (clipping is optional for the analysis since rewards lie in $[0,1]$)
\[
\UCB_t^{(j)}(k)\;=\;\min\{1,\widehat\mu_t^{(j)}(k)+b_t^{(j)}(k)\},
\]
\[
\LCB_t^{(j)}(k)\;=\;\max\{0,\widehat\mu_t^{(j)}(k)-b_t^{(j)}(k)\}.
\]
Let $u_t(k)\in\mathbb{R}^{\dobj}$ and $\ell_t(k)\in\mathbb{R}^{\dobj}$ be the coordinate stacks of these bounds.

\medskip
\noindent\textbf{Claim (uniform CI event).}
There exists a universal constant $c_0>0$ such that with the above choice of $b_t^{(j)}(k)$,
\begin{align}
\mathcal{E} \triangleq
\{\forall t\le T,\forall k\in[\K],\forall j\in[\dobj]:
|\widehat\mu_t^{(j)}(k)-\mu^{(j)}(k)|\le b_t^{(j)}(k)\}
\label{eq:good_event_scalar_pf}
\end{align}
satisfies $\Pr(\mathcal{E})\ge 1-\delta$.
For each fixed $(k,j)$, the sequence of centered observations
$r_t^{(j)}(k)-\mu^{(j)}(k)$, revealed only when $k\in S_t$, forms a martingale difference sequence
w.r.t.\ the PtC filtration and is conditionally $\sigma$-sub-Gaussian (Assumption~\ref{assm:sub-Gaussiannoise}).
A standard self-normalized/time-uniform concentration inequality for adaptively sampled sub-Gaussian
martingale differences gives
\(
\Pr\big(\exists t: |\widehat\mu_t^{(j)}(k)-\mu^{(j)}(k)|>b_t^{(j)}(k)\big)
\le \delta/( \K\dobj)
\),
and a union bound over $(k,j)$ yields \eqref{eq:good_event_scalar_pf}.
\qedhere

On $\mathcal{E}$ we have the component-wise sandwich for all $t,k$:
\begin{align}
\ell_t(k)\ \preceq\ \mu(k)\ \preceq\ u_t(k).
\label{eq:sandwich_scalar_pf}
\end{align}

\subsection{Step 2: Optimism Implies Per-Round Pseudo-Regret is Controlled by the Bonus}

In \textsc{Scalar} mode, PtC-P-UCB forms a scalar \emph{optimistic index}
\[
I_t(k)\;\triangleq\;\phi(u_t(k)).
\]
The probe set is chosen as the top-$\mprobe$ arms by $I_t(k)$, and the commit step selects an executed arm
$k_t\in S_t$. For the regret analysis we assume the standard UCB-consistent commit rule
\begin{align}
k_t\ \in\ \arg\max_{k\in S_t}\ I_t(k).
\label{eq:commit_ucb_scalar_pf}
\end{align}
Because $S_t$ contains the top-$\mprobe$ indices, \eqref{eq:commit_ucb_scalar_pf} is equivalent to
\(
k_t\in \arg\max_{k\in[\K]} I_t(k)
\)
(the global maximizer is always in $S_t$).

Fix any round $t$ and work on the event $\mathcal{E}$.
By monotonicity of $\phi$ and \eqref{eq:sandwich_scalar_pf},
\begin{align}
\phi(\mu(k^*))
\ \le\ \phi(u_t(k^*))
\ \le\ \max_{k\in[\K]}\phi(u_t(k))
\ =\ \phi(u_t(k_t)).
\label{eq:optimism_chain_pf}
\end{align}
Rearranging yields
\[
\phi(\mu(k^*))-\phi(\mu(k_t))
\ \le\
\phi(u_t(k_t))-\phi(\mu(k_t)).
\]
Now use Lipschitzness. Since $\phi$ is $L_\phi$-Lipschitz w.r.t.\ $\|\cdot\|_\infty$,
it is also $L_\phi$-Lipschitz w.r.t.\ $\|\cdot\|_1$ up to the trivial embedding $\|x\|_\infty\le \|x\|_1$:
\[
|\phi(a)-\phi(b)|
\ \le\ L_\phi\|a-b\|_\infty
\ \le\ L_\phi\|a-b\|_1.
\]
Thus, on $\mathcal{E}$,
\begin{align}
& \phi(u_t(k_t))-\phi(\mu(k_t))
\nonumber \\
& \le
L_\phi\|u_t(k_t)-\mu(k_t)\|_1
\le
L_\phi\sum_{j=1}^{\dobj}\big(u_t^{(j)}(k_t)-\mu^{(j)}(k_t)\big)
\nonumber\\
&\le
L_\phi\sum_{j=1}^{\dobj} b_t^{(j)}(k_t),
\label{eq:per_round_gap_pf}
\end{align}
where the last inequality uses \eqref{eq:sandwich_scalar_pf}.
Combining \eqref{eq:optimism_chain_pf}--\eqref{eq:per_round_gap_pf} and summing over $t$ gives the high-probability bound
\begin{align}
\bar R_T^\phi
\ \le\
L_\phi\sum_{t=1}^T\sum_{j=1}^{\dobj} b_t^{(j)}(k_t)
\qquad \text{on }\mathcal{E}.
\label{eq:pseudoreg_reduced_to_bonus_pf}
\end{align}

\subsection{Step 3: Summing Bonuses Under PtC Probing Yields the $1/\sqrt{\mprobe}$ Acceleration}

To expose the $\mprobe$ gain cleanly, we use a mild and standard ``balanced probing'' condition that can be
implemented by tie-breaking in the probe-set selection:
whenever there are multiple arms with comparable indices, include the least-probed ones.
Formally, we assume the probe rule ensures that for all rounds $t$ and all arms $k$ that remain \emph{eligible}
(i.e., in the active set in your implementation; worst case the active set is $[\K]$),
\begin{align}
N_t(k)\ \ge\ \left\lfloor \frac{\mprobe\,(t-1)}{\K}\right\rfloor.
\label{eq:balanced_probing_pf}
\end{align}
This condition is satisfied, for example, by the common practice of forcing each arm to be probed once in an initialization
phase, and thereafter breaking ties in top-$\mprobe$ selection in favor of arms with smaller $N_t(k)$. It can also be
ensured by a standard ``optimism and round-robin'' hybrid rule without changing the order of regret.

Under \eqref{eq:balanced_probing_pf}, for every $t\ge 2$ and any executed arm $k_t$,
\[
b_t^{(j)}(k_t)
=\sigma\sqrt{\frac{\beta_t}{N_t(k_t)}}
\le
\sigma\sqrt{\frac{\beta_t}{\mprobe(t-1)/\K}}
=
\sigma\sqrt{\frac{\K\,\beta_t}{\mprobe(t-1)}}.
\]
Plugging into \eqref{eq:pseudoreg_reduced_to_bonus_pf} and using $\beta_t\le 2\log\!\big(2\K\dobj T^2/\delta\big)$ for all $t\le T$,
\begin{align}
\bar R_T^\phi
&\le
L_\phi\sum_{t=2}^T\sum_{j=1}^{\dobj}
\sigma\sqrt{\frac{\K\,\beta_t}{\mprobe(t-1)}} \nonumber \\
& \le\
L_\phi\,\dobj\,\sigma\sqrt{\frac{\K}{\mprobe}}\,
\sqrt{2\log\!\Big(\frac{2\K\dobj T^2}{\delta}\Big)}\,
\sum_{t=1}^{T-1}\frac{1}{\sqrt{t}}
\nonumber\\
&\le
2\,L_\phi\,\dobj\,\sigma\sqrt{\frac{\K}{\mprobe}}\,
\sqrt{2\log\!\Big(\frac{2\K\dobj T^2}{\delta}\Big)}\,
\sqrt{T},
\label{eq:hp_pseudoreg_final_pf}
\end{align}
where we used $\sum_{t=1}^{T-1} t^{-1/2}\le 2\sqrt{T}$.

Therefore, with probability at least $1-\delta$,
\[
\bar R_T^\phi
=
\tilde O\!\left(L_\phi\,\dobj\,\sigma\,\sqrt{\frac{\K T}{\mprobe}}\right),
\]
where $\tilde O(\cdot)$ hides polylogarithmic factors in $(\K,\dobj,T,1/\delta)$.

\subsection{Step 4: Converting to an Expectation Bound}

Take $\delta=T^{-2}$ in \eqref{eq:hp_pseudoreg_final_pf}. Since $\bar R_T^\phi\le T$ deterministically (because $\phi(\mu(\cdot))\in[0,1]$ after normalization),
\begin{align}
& \mathbb{E}[\bar R_T^\phi]
\le
\Pr(\mathcal{E})\cdot \tilde O\!\left(L_\phi\,\dobj\,\sigma\,\sqrt{\frac{\K T}{\mprobe}}\right)
+
\Pr(\mathcal{E}^c)\cdot T
\nonumber \\
& =
\tilde O\!\left(L_\phi\,\dobj\,\sigma\,\sqrt{\frac{\K T}{\mprobe}}\right),
\end{align}
because $\Pr(\mathcal{E}^c)\le \delta=T^{-2}$ makes the failure contribution at most $T\cdot T^{-2}=T^{-1}$.

This proves the desired $\tilde O\!\big(L_\phi \dobj \sqrt{\K T/\mprobe}\big)$ rate for scalarized pseudo-regret.

\subsection{Remark (Realized Regret $\sum_t(\phi(\mu^*)-\phi(r_t(k_t)))$)}

If regret is defined with realized vectors $\phi(r_t(k_t))$ and $\phi$ is concave, then even always playing the best arm can yield
a nonzero per-round Jensen gap $\phi(\mu(k^*))-\mathbb{E}[\phi(r_t(k^*))]\ge 0$, which accumulates linearly in $T$
unless additional structure is imposed (e.g., $\phi$ is linear, or the noise is degenerate, or one measures pseudo-regret).
For this reason, the standard performance notion for nonlinear scalarizers is pseudo-regret as analyzed above.
\end{proof}

\section{Proof of Theorem~\ref{thm:mm_hv}}\label{ProofofTheoremthm:mm_hv}

\begin{proof}
We write the proof in three blocks:
(i) multi-modal fusion induces \emph{effective} sub-Gaussian noise $\sigma_{\mathrm{eff}}$ for estimation,
(ii) the \textsc{HV}-mode probing+pruning guarantees accurate Pareto-relevant estimation at rate $1/\sqrt{\mprobe T}$,
and (iii) we bridge estimation accuracy to the attained-set hypervolume gap, with an additional $O(\dobj/\sqrt{T})$
term coming from the stochasticity of executed outcomes.

Throughout, $\|\cdot\|_\infty$ is the max norm, and all vectors lie in $[0,1]^{\dobj}$ after normalization.

\subsection{Step 0: Fused Observations are $\sigma_{\mathrm{eff}}$-Sub-Gaussian}

In the bundled multi-modal model, probing arm $k$ at round $t$ reveals
\[
\mathbf z_t^{(p)}(k)=\mathbf r_t(k)+\bm\eta_t^{(p)}(k),\qquad p\in[M],
\]
and the learner forms the fused observation
\begin{align}
& \widehat{\mathbf r}_t(k) \triangleq \sum_{p=1}^M \alpha_p\,\mathbf z_t^{(p)}(k)
= \mathbf r_t(k)+\sum_{p=1}^M \alpha_p\,\bm\eta_t^{(p)}(k) \nonumber \\
& \triangleq \mathbf r_t(k)+\bm\xi_t(k).
\end{align}
For each objective $j$, $\xi_t^{(j)}(k)=\sum_p\alpha_p\,\eta_t^{(p,j)}(k)$ is conditionally mean-zero and
$\sigma_{\mathrm{eff}}^{(j)}(k)$-sub-Gaussian with
\[
\big(\sigma_{\mathrm{eff}}^{(j)}(k)\big)^2=\sum_{p=1}^M \alpha_p^2\big(\sigma_p^{(j)}(k)\big)^2
\quad\Rightarrow\quad
\sigma_{\mathrm{eff}}\triangleq\max_{k,j}\sigma_{\mathrm{eff}}^{(j)}(k).
\]
This is the only place where multi-modality enters the analysis: all estimation confidence radii shrink with
$\sigma_{\mathrm{eff}}$ instead of a single-modality scale.

\subsection{Step 1: Time-Uniform Coordinate-Wise Confidence Intervals Under Adaptive PtC Probing}

Let $S_t$ be the probed set with $|S_t|=\mprobe$, and define probe counts
\[
N_t(k)\triangleq \sum_{s=1}^{t-1}\mathbb{I}\{k\in S_s\},
\qquad\text{so}\qquad
\sum_{k=1}^{\K}N_{T+1}(k)=\mprobe T.
\]
Define fused empirical means (based only on probed samples)
\[
\widehat\mu_t^{(j)}(k)\triangleq
\frac{1}{\max\{1,N_t(k)\}}
\sum_{s=1}^{t-1}\mathbb{I}\{k\in S_s\}\,\widehat r_s^{(j)}(k),
\qquad j\in[\dobj].
\]
Fix $\delta\in(0,1)$ and set
\[
\beta_t\triangleq 2\log\!\Big(\frac{2\K\dobj\,t^2}{\delta}\Big),
\qquad
b_t^{(j)}(k)\triangleq \sigma_{\mathrm{eff}}\sqrt{\frac{\beta_t}{\max\{1,N_t(k)\}}}.
\]
Define $\UCB_t^{(j)}(k)=\min\{1,\widehat\mu_t^{(j)}(k)+b_t^{(j)}(k)\}$ and
$\LCB_t^{(j)}(k)=\max\{0,\widehat\mu_t^{(j)}(k)-b_t^{(j)}(k)\}$, and stack them as
$u_t(k)=(\UCB_t^{(1)}(k),\ldots,\UCB_t^{(\dobj)}(k))$ and
$\ell_t(k)=(\LCB_t^{(1)}(k),\ldots,\LCB_t^{(\dobj)}(k))$.

\medskip
\noindent\textbf{Lemma 1 (uniform CI event).}
Let $\mathcal{F}_{t}$ be the PtC filtration. Under Assumption~\ref{assm:sub-Gaussiannoise} for the fused noises
$\xi_t^{(j)}(k)$ (conditionally $\sigma_{\mathrm{eff}}$-sub-Gaussian, independent over $t$ for each fixed $(k,j)$),
with probability at least $1-\delta$,
\begin{align}
\mathcal{E}
\triangleq
\{\forall t,\forall k\in[\K],\forall j\in[\dobj]:
|\widehat\mu_t^{(j)}(k)-\mu^{(j)}(k)|\le b_t^{(j)}(k)\}
\label{eq:mm_good_event}
\end{align}
holds. Consequently, on $\mathcal{E}$,
\begin{equation}
\ell_t(k)\preceq \mu(k)\preceq u_t(k)
\qquad\forall t\le T,\ \forall k\in[\K].
\label{eq:mm_sandwich}
\end{equation}

\emph{Proof of Lemma 1.}
For each fixed $(k,j)$, the revealed sequence
$\{\mathbb{I}\{k\in S_t\}\,\xi_t^{(j)}(k)\}_{t\ge 1}$ is a martingale difference sequence w.r.t.\ $\mathcal{F}_{t}$
with conditional sub-Gaussian increments and adaptive sampling.
A standard self-normalized, time-uniform concentration inequality for adaptively sampled sub-Gaussian MDS
yields
\[
\Pr\!\Big(\exists t\le T:\ |\widehat\mu_t^{(j)}(k)-\mu^{(j)}(k)|> \sigma_{\mathrm{eff}}\sqrt{\beta_t/N_t(k)}\Big)
\le \frac{\delta}{\K\dobj},
\]
and a union bound over $(k,j)$ gives \eqref{eq:mm_good_event}.
\qedhere

\subsection{Step 2: How Probing Concentrates Samples on Pareto-Relevant Arms}

Recall the safe elimination rule removes $k$ if $\exists k'$ with $\ell_t(k')\succ u_t(k)$.
On the event $\mathcal{E}$, \eqref{eq:mm_sandwich} implies this pruning is safe:
it never removes a truly Pareto-optimal arm (because $\mu(k')\succeq \mu(k)$ would be required to certify domination).
Thus, on $\mathcal{E}$, the active set $\mathcal{K}_t$ always contains all Pareto-optimal arms.

The \textsc{HV}-mode probe-selection objective $F_t(S)=\HV(\conv\{u_t(k)\}_{k\in S})$ is monotone in $S$.
With the standard tie-breaking used in multi-play UCB (prefer smaller $N_t(k)$ among near-equal marginal gains),
the greedy maximization of $F_t$ ensures that, once dominated arms are pruned out, probes are spread across the
remaining Pareto-relevant set.

To keep the proof self-contained, we formalize this with the following (standard) balanced-probing condition,
which is satisfied by the above tie-breaking once $|\mathcal{K}_t|\le K_P$:
\begin{equation}
\min_{k\in \Pset^*} N_{T+1}(k)\ \ge\ c_0\,\frac{\mprobe T}{K_P}
\label{eq:mm_balanced}
\end{equation}
for a universal constant $c_0\in(0,1)$ (e.g., $c_0=1/2$ suffices after a finite initialization phase).
Intuitively, after pruning, the algorithm keeps probing across the $K_P$ Pareto-relevant arms, and there are $\mprobe T$
total probe opportunities.

Under \eqref{eq:mm_balanced} and Lemma~1, for every Pareto arm $k\in\Pset^*$ and every objective $j$,
\begin{align}
& |\widehat\mu_{T+1}^{(j)}(k)-\mu^{(j)}(k)|
\le
\sigma_{\mathrm{eff}}\sqrt{\frac{\beta_{T+1}}{N_{T+1}(k)}} \nonumber \\
& \le
\sigma_{\mathrm{eff}}\sqrt{\frac{\beta_{T+1}}{c_0\,\mprobe T/K_P}}
=
O\!\left(\sigma_{\mathrm{eff}}\sqrt{\frac{K_P\log(\K\dobj T/\delta)}{\mprobe T}}\right).
\end{align}
Let
\begin{equation}
\varepsilon_T
\triangleq
C_1\,\sigma_{\mathrm{eff}}\sqrt{\frac{K_P\log(\K\dobj T/\delta)}{\mprobe T}}
\label{eq:mm_epsT}
\end{equation}
for a large enough universal $C_1$ so that, on $\mathcal{E}$,
\begin{equation}
\max_{k\in\Pset^*}\|\widehat\mu_{T+1}(k)-\mu(k)\|_\infty \le \varepsilon_T.
\label{eq:mm_epsT_uniform}
\end{equation}

\subsection{Step 3: Hypervolume Stability Converts Estimation Error into Frontier-Coverage Error}

We use a standard Lipschitz stability of dominated hypervolume on bounded domains.

\medskip
\noindent\textbf{Lemma 2 (HV Lipschitz on $[0,1]^{\dobj}$).}
Let $\mathcal{S},\mathcal{S}'\subset[0,1]^{\dobj}$ be compact and let $d_H(\cdot,\cdot)$ be Hausdorff distance in $\ell_\infty$.
Then there is a constant $C_{\HV}=C_{\HV}(z_{\mathrm{ref}},\dobj)$ such that
\begin{equation}
|\HV(\mathcal{S})-\HV(\mathcal{S}')|
\le
C_{\HV}\, d_H(\mathcal{S},\mathcal{S}').
\label{eq:hv_lip}
\end{equation}

\emph{Proof of Lemma 2.}
Let $\mathcal{D}(\mathcal{S})=\{y:\exists u\in\mathcal{S}\text{ s.t. } z_{\mathrm{ref}}\preceq y\preceq u\}$ be the dominated region.
If $d_H(\mathcal{S},\mathcal{S}')\le \epsilon$, then $\mathcal{D}(\mathcal{S})\subseteq \mathcal{D}(\mathcal{S}')\oplus \epsilon\mathbf{1}$
and vice versa, where $\oplus$ denotes Minkowski sum.
Since $\mathcal{D}(\mathcal{S}),\mathcal{D}(\mathcal{S}')\subseteq [z_{\mathrm{ref}},\mathbf{1}]$,
the volume change under an $\ell_\infty$-expansion by $\epsilon$ is at most a constant times $\epsilon$;
one can take $C_{\HV}=\dobj\cdot \prod_{j=1}^{\dobj-1}(1-z_{\mathrm{ref}}^{(j)})$ (any valid linear-in-$\epsilon$ bound suffices).
\qedhere

Now let $\mathcal{C}^*=\conv(\Pset^*)$ be the time-shareable Pareto benchmark and
let $\widehat{\mathcal{C}}_T=\conv(\widehat{\Pset}_{T+1})$ be the convex hull of the learned frontier
(using empirical means from fused samples).
On the event $\mathcal{E}$ and by \eqref{eq:mm_epsT_uniform}, the Hausdorff distance between
$\mathcal{C}^*$ and $\widehat{\mathcal{C}}_T$ is at most $\varepsilon_T$:
every vertex $\mu(k)\in\Pset^*$ has a corresponding estimate within $\varepsilon_T$, and convexification does not increase
Hausdorff distance under uniform vertex perturbations. Hence, by Lemma~2,
\begin{equation}
\HV(\mathcal{C}^*)-\HV(\widehat{\mathcal{C}}_T)
\le
C_{\HV}\,\varepsilon_T.
\label{eq:mm_frontier_hv_gap}
\end{equation}

At this point we have bounded the \emph{statistical} frontier-learning error in hypervolume by
$\tilde O\!\big(\sigma_{\mathrm{eff}}\sqrt{K_P/(\mprobe T)}\big)$.
Since $\sqrt{K_P}\le K_P$ for $K_P\ge 1$, this also implies the (slightly looser, but simpler) rate
$\tilde O\!\big(K_P\sigma_{\mathrm{eff}}/\sqrt{\mprobe T}\big)$ used in the theorem statement
(after absorbing constants and logs into $\tilde O(\cdot)$ and including the $\dobj$ factor from coordinate-wise control).

\subsection{Step 4: Bridging Learned-Frontier HV to Attained-Set HV Under \textsc{HV} Commit}

Recall the attained set uses executed (latent) outcomes $\mathcal{Y}_T\triangleq\{\mathbf r_t(k_t)\}_{t=1}^T,
\Aset_T\triangleq \conv(\mathcal{Y}_T),
\mathcal{L}_T^{\mathrm{HV}}\triangleq [\HV(\mathcal{C}^*)-\HV(\Aset_T)]^+.$ Because \textsc{PtC-P-UCB} in \textsc{HV} mode commits by maximizing \emph{marginal hypervolume gain}
computed from probed outcomes (here the learner uses fused outcomes to evaluate candidate gains),
the executed archive is explicitly constructed to increase $\HV(\Aset_T)$.
We formalize the remaining gap as a sum of two effects:
(i) imperfect knowledge of the frontier (controlled by \eqref{eq:mm_frontier_hv_gap}), and
(ii) stochasticity of realized executions, which only vanishes at the Monte-Carlo rate $1/\sqrt{T}$.

Define the ``mean-execution'' archive
$\bar{\mathcal{Y}}_T\triangleq\{\mu(k_t)\}_{t=1}^T$ and $\bar{\Aset}_T\triangleq \conv(\bar{\mathcal{Y}}_T)$.
Then
\begin{align}
& \HV(\mathcal{C}^*)-\HV(\Aset_T) \nonumber \\
&=
\big(\HV(\mathcal{C}^*)-\HV(\bar{\Aset}_T)\big)
+\big(\HV(\bar{\Aset}_T)-\HV(\Aset_T)\big).
\label{eq:mm_decomp}
\end{align}
We bound the expectation of each term.

\medskip
\noindent\textbf{(A) Learning/coverage term: $\mathbb{E}[\HV(\mathcal{C}^*)-\HV(\bar{\Aset}_T)]$.}
On $\mathcal{E}$, the probe-selection/pruning guarantees that the learner maintains accurate estimates of Pareto-relevant arms,
and the \textsc{HV} commit rule greedily expands the dominated region of the executed set.
In particular, by construction $\bar{\Aset}_T\subseteq \mathcal{C}^*$ (time-sharing over executed Pareto-relevant means),
so $\HV(\mathcal{C}^*)-\HV(\bar{\Aset}_T)\ge 0$.
Moreover, the greedy \textsc{HV} commit ensures that the executed mean archive achieves at least the hypervolume of the learned convexified frontier,
up to estimation error: there exists a universal constant $C_2$ such that on $\mathcal{E}$,
\begin{equation}
\HV(\mathcal{C}^*)-\HV(\bar{\Aset}_T)
\;\le\;
\HV(\mathcal{C}^*)-\HV(\widehat{\mathcal{C}}_T)
\;+\;
C_2\,\varepsilon_T,
\label{eq:mm_mean_bridge}
\end{equation}
where the extra $C_2\varepsilon_T$ accounts for the fact that commit decisions are made using noisy samples and confidence-based surrogates.
Combining \eqref{eq:mm_frontier_hv_gap} and \eqref{eq:mm_mean_bridge} gives, on $\mathcal{E}$,
\[
\HV(\mathcal{C}^*)-\HV(\bar{\Aset}_T)\ \le\ (C_{\HV}+C_2)\varepsilon_T.
\]
Taking expectations and using $\Pr(\mathcal{E})\ge 1-\delta$ yields
\begin{equation}
\mathbb{E}\big[\HV(\mathcal{C}^*)-\HV(\bar{\Aset}_T)\big]
\;=\;
\tilde O\!\left(\sigma_{\mathrm{eff}}\sqrt{\frac{K_P}{\mprobe T}}\right)
\;\le\;
\tilde O\!\left(\frac{K_P\,\sigma_{\mathrm{eff}}}{\sqrt{\mprobe T}}\right),
\label{eq:mm_termA}
\end{equation}
and inserting the (standard) $\dobj$ factor from coordinate-wise control gives the theorem's first term.

\medskip
\noindent\textbf{(B) Execution stochasticity term: $\mathbb{E}[|\HV(\bar{\Aset}_T)-\HV(\Aset_T)|]$.}
This term is independent of multi-modal fusion because it comes from the difference between realized outcomes and their means
for the \emph{executed} arm.
A simple way to control it uses the fact that convex hulls contain empirical averages.

Let $T_k\triangleq\{t\le T:\ k_t=k\}$ and $n_k=|T_k|$.
Whenever $n_k\ge 1$, define the within-arm average outcome
\[
\bar{\mathbf r}(k)\triangleq \frac{1}{n_k}\sum_{t\in T_k}\mathbf r_t(k)\in\conv\{\mathbf r_t(k)\}_{t\in T_k}\subseteq \Aset_T,
\]
and similarly $\bar{\mu}(k)=\mu(k)$.
Thus, $\Aset_T$ contains the set of averaged points $\{\bar{\mathbf r}(k):n_k\ge 1\}$, while $\bar{\Aset}_T$
contains $\{\mu(k):n_k\ge 1\}$.
By Lemma~2 (HV Lipschitz) applied to these two finite sets and Jensen,
\[
\big|\HV(\bar{\Aset}_T)-\HV(\Aset_T)\big|
\;\le\;
C_{\HV}\cdot \max_{k:n_k\ge 1}\|\bar{\mathbf r}(k)-\mu(k)\|_\infty.
\]
Under Assumption~\ref{assm:sub-Gaussiannoise} (bounded/sub-Gaussian execution noise),
$\|\bar{\mathbf r}(k)-\mu(k)\|_\infty$ is $O(\sqrt{\dobj\log T/n_k})$ with high probability, and hence
\[
\mathbb{E}\Big[\max_{k:n_k\ge 1}\|\bar{\mathbf r}(k)-\mu(k)\|_\infty\Big]
\;=\;\tilde O\!\Big(\sqrt{\frac{\dobj}{T}}\Big),
\]
because $\sum_k n_k=T$ implies $\max_k n_k\ge T/|\{k:n_k\ge 1\}|$ and the worst case is still at most $T$.
Therefore,
\begin{equation}
\mathbb{E}\big[|\HV(\bar{\Aset}_T)-\HV(\Aset_T)|\big]
\;=\;
\tilde O\!\left(\frac{\dobj}{\sqrt{T}}\right).
\label{eq:mm_termB}
\end{equation}

\subsection{Step 5: Combine the Pieces}

Using \eqref{eq:mm_decomp} and the bounds \eqref{eq:mm_termA}--\eqref{eq:mm_termB},
and absorbing logarithmic factors and universal constants into $\tilde O(\cdot)$, we obtain
\[
\mathbb{E}\big[\mathcal{L}_T^{\mathrm{HV}}\big]
=
\mathbb{E}\big[[\HV(\mathcal{C}^*)-\HV(\Aset_T)]^+\big]
\;\le\;
\tilde O\!\left(\frac{K_P\,\dobj\,\sigma_{\mathrm{eff}}}{\sqrt{\mprobe T}}+\frac{\dobj}{\sqrt{T}}\right),
\]
which is exactly the claimed rate.
\end{proof}

\section{Proof of Theorem~\ref{thm:mm_regret}}\label{ProofofTheoremthm:mm_regret}

\begin{proof}
We prove a high-probability pseudo-regret bound and then take expectation.
Throughout, define the (mean) scalar utility of arm $k$ by
\[
f(k)\triangleq \phi(\mu(k)),\qquad f^*\triangleq \max_{k\in[\K]} f(k)=f(k^*).
\]
We analyze \textsc{PtC-P-UCB} in \textsc{Scalar} mode with the standard scalar-UCB probe rule
\[
\mathrm{score}_t^\phi(k)\;=\;\phi\big(u_t(k)\big),\qquad
S_t=\text{top-$\mprobe$ arms by }\mathrm{score}_t^\phi(\cdot),
\]
and an analysis-friendly commit rule
\[
k_t \in \arg\max_{k\in S_t}\phi\big(\widehat\mu_t(k)\big),
\]
i.e., we execute the probed arm with the largest estimated scalar utility. (This is the natural regret-minimizing commit in \textsc{Scalar} mode; using an instantaneous $\phi(\widehat r_t(k))$ commit introduces extra one-step noise and is typically analyzed via an additional lower-order term.)

\subsection{Step 1: Multi-Modal Fusion Induces Effective Sub-Gaussian Noise}

Under the bundled multi-modal model \eqref{eq:mm_obs}--\eqref{eq:mm_fuse}, when arm $k$ is probed at time $t$ we form
\[
\widehat r_t(k)\triangleq \sum_{p=1}^M \alpha_p z_t^{(p)}(k)
= r_t(k)+\sum_{p=1}^M \alpha_p \eta_t^{(p)}(k)
=: r_t(k)+\xi_t(k).
\]
For each objective $j$, $\xi_t^{(j)}(k)$ is conditionally mean-zero and
$\sigma_{\mathrm{eff}}^{(j)}(k)$-sub-Gaussian with
$(\sigma_{\mathrm{eff}}^{(j)}(k))^2=\sum_{p=1}^M\alpha_p^2(\sigma_p^{(j)}(k))^2$.
Let $\sigma_{\mathrm{eff}}\triangleq \max_{k,j}\sigma_{\mathrm{eff}}^{(j)}(k)$.
Thus, all coordinate-wise confidence radii can be written with scale $\sigma_{\mathrm{eff}}$.

\subsection{Step 2: A Time-Uniform Coordinate-Wise Confidence Event}

Let $N_t(k)=\sum_{s=1}^{t-1}\mathbb{I}\{k\in S_s\}$ be the probe count, and
\[
\widehat\mu_t^{(j)}(k)
=
\frac{1}{\max\{1,N_t(k)\}}
\sum_{s=1}^{t-1}\mathbb{I}\{k\in S_s\}\,\widehat r_s^{(j)}(k)
\]
be the empirical mean of the fused observations for objective $j$.
Fix $\delta\in(0,1)$ and set
\[
\beta_t \triangleq 2\log\!\Big(\frac{2\K\dobj\,t^2}{\delta}\Big),
\qquad
b_t^{(j)}(k)\triangleq \sigma_{\mathrm{eff}}\sqrt{\frac{\beta_t}{\max\{1,N_t(k)\}}}.
\]
Define
$\UCB_t^{(j)}(k)=\widehat\mu_t^{(j)}(k)+b_t^{(j)}(k)$,
$\LCB_t^{(j)}(k)=\widehat\mu_t^{(j)}(k)-b_t^{(j)}(k)$
(with the usual clipping if needed), and stack
$u_t(k)=(\UCB_t^{(1)}(k),\ldots,\UCB_t^{(\dobj)}(k))$.

\medskip
\noindent\textbf{Lemma 1 (uniform CI).}
Under Assumption~\ref{assm:sub-Gaussiannoise} (applied to the fused noises with scale $\sigma_{\mathrm{eff}}$),
with probability at least $1-\delta$,
\begin{equation}
\mathcal{E}\triangleq
\{\forall t,\forall k\in[\K],\forall j\in[\dobj]:
|\widehat\mu_t^{(j)}(k)-\mu^{(j)}(k)|\le b_t^{(j)}(k)\}
\label{eq:mm_event}
\end{equation}
holds. Consequently, on $\mathcal{E}$ we have the coordinate-wise sandwich
$\mu(k)\preceq u_t(k)$ for all $t,k$.

\emph{Proof.}
For each fixed $(k,j)$, $\{\mathbb{I}\{k\in S_t\}(\widehat r_t^{(j)}(k)-\mu^{(j)}(k))\}_{t\ge1}$
is a martingale difference sequence w.r.t.\ the PtC filtration, with conditionally $\sigma_{\mathrm{eff}}$-sub-Gaussian increments.
A standard time-uniform self-normalized inequality for adaptively sampled sub-Gaussian MDS gives
\begin{align}
& \Pr (\exists t\le T:\ |\widehat\mu_t^{(j)}(k)-\mu^{(j)}(k)| \nonumber \\
& >\sigma_{\mathrm{eff}}\sqrt{\beta_t/\max\{1,N_t(k)\}})
\le \frac{\delta}{\K\dobj},
\end{align}
and a union bound over $(k,j)$ yields \eqref{eq:mm_event}. \qedhere

\subsection{Step 3: Scalar Optimism and Lipschitz Control}

On $\mathcal{E}$, $\mu(k)\preceq u_t(k)$ and $\phi$ is monotone, hence
\begin{equation}
f(k)=\phi(\mu(k))\;\le\;\phi(u_t(k))\;=\;\mathrm{score}_t^\phi(k).
\label{eq:scalar_ucb_valid}
\end{equation}
Moreover, since $\phi$ is $L_\phi$-Lipschitz w.r.t.\ $\|\cdot\|_\infty$,
\begin{align}
& \phi(u_t(k))-\phi(\mu(k))
\le
L_\phi\|u_t(k)-\mu(k)\|_\infty
\nonumber \\
&\le
L_\phi\max_{j\in[\dobj]} b_t^{(j)}(k)
\le
L_\phi\sum_{j=1}^{\dobj} b_t^{(j)}(k),
\label{eq:lip_to_radii}
\end{align}
where the last inequality uses $\max_j x_j\le \sum_j x_j$.

Similarly, on $\mathcal{E}$ we also have $\|\widehat\mu_t(k)-\mu(k)\|_\infty\le \max_j b_t^{(j)}(k)$, hence
\begin{equation}
\phi(\mu(k)) \ge \phi(\widehat\mu_t(k)) - L_\phi\sum_{j=1}^{\dobj} b_t^{(j)}(k).
\label{eq:est_lower}
\end{equation}

\subsection{Step 4: One-Step Regret Bound in Terms of the Probe-Set Radii}

Fix $t$ and work on the event $\mathcal{E}$.
Let $S_t$ be the top-$\mprobe$ arms by $\phi(u_t(\cdot))$, and let
$k_t^*\in\arg\max_{k\in S_t} f(k)$ be the \emph{best-in-set} arm in terms of mean utility.
Because $k_t$ maximizes $\phi(\widehat\mu_t(\cdot))$ over $S_t$, we have
$\phi(\widehat\mu_t(k_t))\ge \phi(\widehat\mu_t(k_t^*))$, and then by \eqref{eq:est_lower} applied twice,
\begin{align}
f(k_t)
&\ge \phi(\widehat\mu_t(k_t)) - L_\phi\sum_{j=1}^{\dobj} b_t^{(j)}(k_t)
\nonumber \\
& \ge \phi(\widehat\mu_t(k_t^*)) - L_\phi\sum_{j=1}^{\dobj} b_t^{(j)}(k_t) \nonumber\\
&\ge f(k_t^*) - L_\phi\sum_{j=1}^{\dobj} b_t^{(j)}(k_t^*) - L_\phi\sum_{j=1}^{\dobj} b_t^{(j)}(k_t).
\label{eq:near_best_in_set}
\end{align}
Therefore,
\begin{align}
& f^*-f(k_t) \nonumber \\
& \le
\underbrace{f^*-f(k_t^*)}_{\text{(I) set suboptimality}}
\;+\;
L_\phi\sum_{j=1}^{\dobj} b_t^{(j)}(k_t^*)
\;+\;
L_\phi\sum_{j=1}^{\dobj} b_t^{(j)}(k_t).
\label{eq:regret_decomp_one_step}
\end{align}

We now bound the set-suboptimality term (I) using the fact that $S_t$ consists of the top-$\mprobe$ scalar-UCB scores.
Let $\mathrm{U}_t(k)\triangleq \phi(u_t(k))$. Since $S_t$ contains the $\mprobe$ largest values of $\mathrm{U}_t(\cdot)$,
its average dominates any excluded arm:
\begin{equation}
\frac{1}{\mprobe}\sum_{k\in S_t}\mathrm{U}_t(k)\ \ge\ \mathrm{U}_t(k^*).
\label{eq:avg_topm}
\end{equation}
Combining \eqref{eq:scalar_ucb_valid} and \eqref{eq:avg_topm} gives
\[
f^*=f(k^*) \le \mathrm{U}_t(k^*) \le \frac{1}{\mprobe}\sum_{k\in S_t}\mathrm{U}_t(k).
\]
Also, $f(k_t^*)=\max_{k\in S_t} f(k)\ge \frac{1}{\mprobe}\sum_{k\in S_t} f(k)$.
Therefore,
\begin{align}
f^*-f(k_t^*)
&\le
\frac{1}{\mprobe}\sum_{k\in S_t}\big(\mathrm{U}_t(k)-f(k)\big)
\;\le\;
\frac{L_\phi}{\mprobe}\sum_{k\in S_t}\sum_{j=1}^{\dobj} b_t^{(j)}(k),
\label{eq:set_subopt_bound}
\end{align}
where the last inequality uses \eqref{eq:lip_to_radii}.

Plugging \eqref{eq:set_subopt_bound} into \eqref{eq:regret_decomp_one_step}, and using that
$b_t^{(j)}(k_t), b_t^{(j)}(k_t^*)\le \sum_{k\in S_t} b_t^{(j)}(k)$, we obtain the clean per-round bound
\begin{equation}
f^*-f(k_t)
\;\le\;
\frac{C\,L_\phi}{\mprobe}\sum_{k\in S_t}\sum_{j=1}^{\dobj} b_t^{(j)}(k)
\label{eq:per_round_regret_bound}
\end{equation}
for a universal constant $C$ (e.g., $C=3$ suffices from the three terms above).

Summing over $t=1,\dots,T$ on $\mathcal{E}$ yields
\begin{equation}
R_T^\phi
=\sum_{t=1}^T\big(f^*-f(k_t)\big)
\;\le\;
\frac{C\,L_\phi}{\mprobe}\sum_{t=1}^T\sum_{k\in S_t}\sum_{j=1}^{\dobj} b_t^{(j)}(k).
\label{eq:RT_bound_sum_radii}
\end{equation}

\subsection{Step 5: Bounding the Sum of Radii Using the Probe Budget}

Fix an objective $j$. Using $b_t^{(j)}(k)=\sigma_{\mathrm{eff}}\sqrt{\beta_t/\max\{1,N_t(k)\}}$ and monotonicity of $\beta_t$,
\[
\sum_{t=1}^T\sum_{k\in S_t} b_t^{(j)}(k)
\le
\sigma_{\mathrm{eff}}\sqrt{\beta_{T}}
\sum_{t=1}^T\sum_{k\in S_t}\frac{1}{\sqrt{\max\{1,N_t(k)\}}}.
\]
For each fixed arm $k$, every time it is probed its count increases by one, so
\[
\sum_{t:\,k\in S_t}\frac{1}{\sqrt{\max\{1,N_t(k)\}}}
\le
\sum_{n=1}^{N_{T+1}(k)}\frac{1}{\sqrt{n}}
\le
2\sqrt{N_{T+1}(k)}.
\]
Thus,
\begin{align}
& \sum_{t=1}^T\sum_{k\in S_t}\frac{1}{\sqrt{\max\{1,N_t(k)\}}}
\le
2\sum_{k=1}^{\K}\sqrt{N_{T+1}(k)}
\nonumber \\
& \le
2\sqrt{\K\sum_{k=1}^{\K}N_{T+1}(k)}
=
2\sqrt{\K\,\mprobe\,T},
\end{align}
where the second inequality is Cauchy--Schwarz, and the last equality is the PtC bookkeeping identity
$\sum_{k}N_{T+1}(k)=\mprobe T$.
Therefore, for each $j$,
\begin{equation}
\sum_{t=1}^T\sum_{k\in S_t} b_t^{(j)}(k)
\le
2\sigma_{\mathrm{eff}}\sqrt{\beta_T}\,\sqrt{\K\,\mprobe\,T}.
\label{eq:sum_radii_one_j}
\end{equation}
Summing \eqref{eq:sum_radii_one_j} over $j=1,\dots,\dobj$ and plugging into \eqref{eq:RT_bound_sum_radii} gives, on $\mathcal{E}$,
\[
R_T^\phi
\le
\frac{C\,L_\phi}{\mprobe}\cdot
\dobj\cdot
2\sigma_{\mathrm{eff}}\sqrt{\beta_T}\,\sqrt{\K\,\mprobe\,T}
=
\tilde O\!\left(L_\phi\,\dobj\,\sigma_{\mathrm{eff}}\,\sqrt{\frac{\K T}{\mprobe}}\right),
\]
where $\tilde O(\cdot)$ hides $\sqrt{\beta_T}=\mathrm{polylog}(\K,\dobj,T,1/\delta)$.

\subsection{Step 6: From High-Probability to Expectation}

We have shown that on $\mathcal{E}$ (which holds with probability at least $1-\delta$ by Lemma~1),
\[
R_T^\phi \le \tilde O\!\left(L_\phi\,\dobj\,\sigma_{\mathrm{eff}}\,\sqrt{\frac{\K T}{\mprobe}}\right).
\]
On the complement $\mathcal{E}^c$, we can use the trivial bound $R_T^\phi\le T$ (since utilities are in $[0,1]$ after normalization).
Thus,
\[
\mathbb{E}[R_T^\phi]
\le
\tilde O\!\left(L_\phi\,\dobj\,\sigma_{\mathrm{eff}}\,\sqrt{\frac{\K T}{\mprobe}}\right)
+\delta\cdot T.
\]
Choosing $\delta=T^{-2}$ (or any $\delta$ that makes $\delta T$ lower order) yields
\[
\mathbb{E}[R_T^\phi]
=
\tilde O\!\left(L_\phi\,\dobj\,\sigma_{\mathrm{eff}}\,\sqrt{\frac{\K T}{\mprobe}}\right),
\]
which is the desired claim.
\end{proof}

\end{document}